  \newcommand{\OCMIX}{Fedorova/etal:2013}

\newcommand{\zzrelax}[1]{\relax}

\newlength{\picturewidth}
  \documentclass[10pt]{article}
  \usepackage{amsmath,amsthm,amsfonts,amssymb,latexsym,graphicx,url,stmaryrd,algorithm,algorithmic,enumitem}

\allowdisplaybreaks[4]

\newcounter{protocol}
\newcounter{temporary}
  

\newcommand{\dd}[1]{\,\textrm d{#1}}    % differential in integrals

\DeclareMathOperator{\III}{\boldsymbol{1}}   % indicator function

\DeclareMathOperator{\Prob}{\mathbb{P}}
\DeclareMathOperator{\Expect}{\mathbb{E}}
\DeclareMathOperator{\NNN}{N}   % cumulative size
\let\OE\relax                   % \OE is already defined; forget it
\DeclareMathOperator{\OE}{OE}   % cumulative observed excess
\DeclareMathOperator{\SSS}{S}   % cumulative sum of p-values
\DeclareMathOperator{\OF}{OF}   % cumulative observed fuzziness

\DeclareMathOperator{\MMM}{M}   % cumulative number of multiple predictions
\DeclareMathOperator{\EEE}{E}   % cumulative excess
\DeclareMathOperator{\OM}{OM}   % cumulative number of observed multiple predictions
\DeclareMathOperator{\UUU}{U}   % cumulative unconfidence
\DeclareMathOperator{\FFF}{F}   % cumulative fuzziness
\DeclareMathOperator{\OU}{OU}   % cumulative observed unconfidence

\DeclareMathOperator{\CP}{CP}   % CP (conditional probability)
\DeclareMathOperator{\SP}{SP}   % SP (signed predictability)
\DeclareMathOperator{\MCP}{MCP} % MCP (modified conditional probability)
\DeclareMathOperator{\MSP}{MSP} % MSP (modified signed predictability)

\newcommand{\lc}{_{{\rm lc}}}   % label-conditional

  \newtheorem{lemma}{Lemma}
  
  \newtheorem{corollary}{Corollary}
  \newtheorem{theorem}{Theorem}
  
  \theoremstyle{definition}
  \newtheorem*{remark}{Remark}

  \title{Criteria of efficiency for conformal prediction\thanks{A preliminary version of this paper
    was published as Working Paper~11 of the On-line Compression Modelling project (New Series),
    \texttt{http://alrw.net}, in April 2014.}}
  \author{Vladimir Vovk, Ilia Nouretdinov, Valentina Fedorova,\\Ivan Petej, and Alex Gammerman\\[2mm]
  \texttt{\{volodya.vovk,alushaf,ivan.petej\}{\rm@}gmail.com}\\
  \texttt{\{ilia,alex\}{\rm@}cs.rhul.ac.uk}}

\begin{document}
  \maketitle

\begin{abstract}
  We study optimal conformity measures for various criteria of efficiency of classification
  in an idealised setting.
  This leads to an important class of criteria of efficiency that we call probabilistic;
  it turns out that the most standard criteria of efficiency used in literature on conformal prediction
  are not probabilistic unless the problem of classification is binary.
  We consider both unconditional and label-conditional conformal prediction.

    \bigskip

    \noindent
    The conference version of this paper has been published in the Proceedings of COPA 2016.
\end{abstract}

\section{Introduction}
% \label{sec:introduction}

Conformal prediction is a method of generating prediction sets
that are guaranteed to have a prespecified coverage probability;
in this sense conformal predictors have guaranteed validity.
Different conformal predictors, however, widely differ in their efficiency,
by which we mean the narrowness, in some sense, of their prediction sets.
Empirical investigation of the efficiency of various conformal predictors
is becoming a popular area of research:
see, e.g., \cite{Bala/etal:2014,AMAI:2015}
(and the COPA Proceedings, 2012--2016).
This paper points out that the standard criteria of efficiency used in literature
have a serious disadvantage,
and we define a class of criteria of efficiency, called ``probabilistic'',
that do not share this disadvantage.
In two recent papers
\cite{\OCMIX,Johansson/etal:2013}
two probabilistic criteria have been introduced,
and in this paper we introduce two more
and argue that probabilistic criteria should be used in place of more standard ones.
We concentrate on the case of classification only (the label space is finite).

Surprisingly few criteria of efficiency have been used in literature,
and even fewer have been studied theoretically.
We can speak of the efficiency of individual predictions
or of the overall efficiency of predictions on a test sequence;
the latter is usually (in particular, in this paper) defined by averaging
the efficiency over the individual test examples,
and so in this introductory section we only discuss the former.
This section assumes that the reader knows the basic definitions
of the theory of conformal prediction,
but they will be given in Section~\ref{sec:criteria}
(and Section~\ref{sec:label-conditional} for the label-conditional version),
which can be consulted now.

The two criteria for efficiency of a prediction that have been used most often in literature
(in, e.g., the references given above) are:
\begin{itemize}
\item
  The confidence and credibility of the prediction
  (see, e.g., \cite{Vovk/etal:2005book}, p.~96;
  introduced in \cite{Saunders/etal:1999}).
  This criterion does not depend on the choice of a significance level~$\epsilon$.
\item
  Whether the prediction is a singleton (the ideal case),
  multiple (an inefficient prediction),
  or empty (a superefficient prediction)
  at a given significance level~$\epsilon$.
  This criterion was introduced in \cite{Melluish/etal:2001}, Section~7.2,
  and used extensively in \cite{Vovk/etal:2005book}.
\end{itemize}
The other two criteria that had been used
before the publication of the conference version \cite{Vovk/etal:2016COPA} of this paper
are the sum of the p-values for all potential labels
(this does not depend on the significance level)
and the size of the prediction set at a given significance level:
see the papers
\cite{\OCMIX} and \cite{Johansson/etal:2013}.

In this paper we introduce six other criteria of efficiency:
see Section~\ref{sec:criteria}.
We then discuss (in Sections~\ref{sec:optimal}--\ref{sec:not-probabilistic})
the conformity measures that optimise each of the ten criteria
when the data-generating distribution is known;
this sheds light on the kind of behaviour implicitly encouraged by the criteria
even in the realistic case where the data-generating distribution is unknown.
As we point out in Section~\ref{sec:not-probabilistic},
probabilistic criteria of efficiency are conceptually similar to ``proper scoring rules''
in probability forecasting \cite{Dawid:ESS2006PF,Gneiting/Raftery:2007},
and this is our main motivation for their detailed study in this paper.
In Section~\ref{sec:proofs} we prove the results of Section~\ref{sec:not-probabilistic}.
After that we briefly illustrate the empirical behaviour of two of the criteria
for standard conformal predictors and a benchmark data set
(Section~\ref{sec:empirical}).
Sections~\ref{sec:criteria}--\ref{sec:empirical} discuss the most standard unconditional conformal predictors.
Section~\ref{sec:label-conditional} defines label-conditional conformal predictors
and discusses the analogues of the results of the previous sections for label-conditional predictors.
Finally,
Section~\ref{sec:conclusion} gives some directions of further research.

A version (with a different treatment of empty observations)
of one of the new non-probabilistic criteria of efficiency that we discuss in this paper
(the one that we call the E criterion) has been introduced independently in \cite{Sadinle/etal:2016}.

We only consider the case of randomised (``smoothed'') conformal predictors:
the case of deterministic predictors may lead to combinatorial problems
without an explicit \label{p:combinatorial-1}solution
(this is the case, e.g., for the N criterion defined below).
The situation here is analogous to the Neyman--Pearson lemma:
cf.\ \cite{Lehmann:1986}, Section~3.2.

\section{Criteria of Efficiency for Conformal Predictors and Transducers}
\label{sec:criteria}

Let $\mathbf{X}$ be a measurable space (the \emph{object space})
and $\mathbf{Y}$ be a finite set equipped with the discrete $\sigma$-algebra
(the \emph{label space});
the \emph{example space} is defined to be $\mathbf{Z}:=\mathbf{X}\times\mathbf{Y}$.
We will always assume that the label space $\mathbf{Y}$ is non-empty,
and will usually assume that its size is at least~2.
A \emph{conformity measure} is a measurable function $A$ that assigns to every finite sequence
$(z_1,\ldots,z_n)\in\mathbf{Z}^*$ of examples
a same-length sequence $(\alpha_1,\ldots,\alpha_n)$ of real numbers
and that is equivariant with respect to permutations:
for any $n$ and any permutation $\pi$ of $\{1,\ldots,n\}$,
\begin{equation*}
  (\alpha_1,\ldots,\alpha_n)
  =
  A(z_1,\ldots,z_n)
  \Longrightarrow % \\
  \left(\alpha_{\pi(1)},\ldots,\alpha_{\pi(n)}\right)
  =
  A\left(z_{\pi(1)},\ldots,z_{\pi(n)}\right).
\end{equation*}
The \emph{conformal predictor} determined by $A$ is defined by
\begin{equation}\label{eq:conformal-predictor}
  \Gamma^{\epsilon}(z_1,\ldots,z_l,x)
  =
  \Gamma^{\epsilon}(z_1,\ldots,z_l,x,\tau)
  :=
  \left\{
    y
    \mid
    p^y>\epsilon
  \right\},
\end{equation}
where $(z_1,\ldots,z_l)\in\mathbf{Z}^*$ is a training sequence,
$x$ is a test object,
$\epsilon\in(0,1)$ is a given \emph{significance level},
for each $y\in\mathbf{Y}$ the corresponding \emph{p-value} $p^y$ is defined by
\begin{multline}\label{eq:p}
  p^y
  =
  p^y(z_1,\ldots,z_l,x_{l+1})
  :=
  \frac{1}{l+1}
  \left|\left\{i=1,\ldots,l+1\mid\alpha^y_i<\alpha^y_{l+1}\right\}\right|\\
  +
  \frac{\tau}{l+1}
  \left|\left\{i=1,\ldots,l+1\mid\alpha^y_i=\alpha^y_{l+1}\right\}\right|,
\end{multline}
$\tau$ is a random number distributed uniformly on the interval $[0,1]$
(even conditionally on all the examples),
and the corresponding sequence of \emph{conformity scores} is defined by
\begin{equation}\label{eq:conformity-scores}
  (\alpha_1^y,\ldots,\alpha_l^y,\alpha_{l+1}^y)
  :=
  A(z_1,\ldots,z_l,(x,y)).
\end{equation}
Notice that the system of \emph{prediction sets} \eqref{eq:conformal-predictor}
output by a conformal predictor is decreasing in $\epsilon$, or \emph{nested}.

The \emph{conformal transducer} determined by $A$
outputs the system of p-values $(p^y\mid y\in\mathbf{Y})$
defined by \eqref{eq:p}
for each training sequence $(z_1,\ldots,z_l)$ of examples and each test object $x$.
(This is just a different representation of the conformal predictor.)

Notice that the p-values \eqref{eq:p}
(and, therefore, the corresponding conformal predictors and transducers)
only depend on the \emph{conformity order} corresponding to the given conformity measure:
namely, on the way that the elements of a sequence $(z_1,\ldots,z_n)$
are ordered by the values $(\alpha_1,\ldots,\alpha_n)$
(with $z_i\preceq z_j$ defined to be $\alpha_i\le\alpha_j$).
Therefore, to define conformal predictors and transducers
we may define their conformity orders rather than conformity measures.

The standard property of validity for conformal transducers
is that the p-values $p^y$ are distributed uniformly on $[0,1]$
when the examples $z_1,\ldots,z_l,(x,y)$ are generated independently
from the same probability distribution $Q$ on $\mathbf{Z}$
and $\tau$ is generated independently from the uniform probability distribution on $[0,1]$
(see, e.g., \cite{Vovk/etal:2005book}, Proposition~2.8).
This implies that the probability of error,
$y\notin\Gamma^{\epsilon}(z_1,\ldots,z_l,x)$,
for conformal predictors
is $\epsilon$ at any significance level~$\epsilon$.

Suppose we are given a test sequence $(z_{l+1},\ldots,z_{l+k})$
and would like to use it to measure the efficiency of the predictions
derived from the training sequence $(z_1,\ldots,z_l)$.
(Informally, by the efficiency of conformal predictors
we mean that the prediction sets they output tend to be small,
and by the efficiency of conformal transducers we mean that the p-values they output tend to be small.)
For each test example $z_i=(x_i,y_i)$,
$i=l+1,\ldots,l+k$,
we have a nested family $(\Gamma_i^{\epsilon}\mid\epsilon\in(0,1))$
of subsets of $\mathbf{Y}$,
where
\[
  \Gamma_i^{\epsilon}
  :=
  \Gamma^{\epsilon}(z_1,\ldots,z_l,x_i),
\]
and a system of p-values $(p^y_i\mid y\in\mathbf{Y})$,
where
\[
  p_i^y
  :=
  p^y(z_1,\ldots,z_l,x_i).
\]
In this paper we will discuss ten criteria of efficiency
for such a family or a system,
but some of them will depend, additionally,
on the observed label $y_i$ of the test example.
We start from the \emph{prior} criteria, which do not depend on the observed test labels.

\subsection{Basic criteria}

We will discuss two kinds of criteria:
those applicable to the prediction sets $\Gamma_i^{\epsilon}$
and so depending on the significance level $\epsilon$
and those applicable to systems of p-values $(p_i^y\mid y\in\mathbf{Y})$
and so independent of $\epsilon$.
The simplest criteria of efficiency are:
\begin{itemize}
\item
  The \emph{S criterion}
  (with ``S'' standing for ``sum'')
  measures efficiency by the average sum
  \begin{equation}\label{eq:S-real}
    \frac{1}{k}\sum_{i=l+1}^{l+k}\sum_{y}p^y_i
  \end{equation}
  of the p-values;
  small values are preferable for this criterion.
  It is $\epsilon$-free.
\item
  The \emph{N criterion}
  uses the average size
  \[\frac1k\sum_{i=l+1}^{l+k}\left|\Gamma_i^{\epsilon}\right|\]
  of the prediction sets
  (``N'' stands for ``number'':
  the size of a prediction set is the number of labels in it).
  Small values are preferable.
  Under this criterion the efficiency
  is a function of the significance level $\epsilon$.
\end{itemize}
Both these criteria are prior.
The S criterion was introduced in \cite{\OCMIX}
and the N criterion was introduced independently
in \cite{Johansson/etal:2013} and \cite{\OCMIX},
although the analogue of the N criterion for regression
(where the size of a prediction set is defined to be its Lebesgue measure)
had been used earlier in \cite{Lei/Wasserman:2013}
(whose arXiv version was published in 2012).

\subsection{Other prior criteria}

A disadvantage of the basic criteria is that they look too stringent.
Even for a very efficient conformal transducer,
we cannot expect all p-values $p^y$ to be small:
the p-value corresponding to the true label will not be small with high probability;
and even for a very efficient conformal predictor
we cannot expect the size of its prediction set to be zero:
with high probability it will contain the true label.
The other prior criteria are less stringent.
The ones that do not depend on the significance level are:
\begin{itemize}
\item
  The \emph{U criterion} (with ``U'' standing for ``unconfidence'')
  uses the average unconfidence
  \begin{equation}\label{eq:U}
    \frac{1}{k}
    \sum_{i=l+1}^{l+k}
    \min_y
    \max_{y'\ne y}
    p_i^{y'}
  \end{equation}
  over the test sequence,
  where the \emph{unconfidence} for a test object $x_i$
  is the second largest p-value $\min_y\max_{y'\ne y}p_i^{y'}$;
  small values of \eqref{eq:U} are preferable.
  The U criterion in this form was introduced in \cite{\OCMIX},
  but it is equivalent to using the average confidence (one minus unconfidence),
  which is very common.
  If two conformal transducers have the same average unconfidence,
  the criterion compares the average credibilities
  \begin{equation}\label{eq:credibility}
    \frac{1}{k}
    \sum_{i=l+1}^{l+k}
    \max_{y}
    p_i^{y},
  \end{equation}
  where the \emph{credibility} for a test object $x_i$
  is the largest p-value $\max_{y}p_i^y$;
  smaller values of \eqref{eq:credibility} are preferable.
  (Intuitively, a small credibility is a warning that the test object is unusual,
  and since such a warning presents useful information
  and the probability of a warning is guaranteed to be small,
  we want to be warned as often as possible.)
\item
  The \emph{F criterion} uses the average fuzziness
  \begin{equation}\label{eq:F}
    \frac{1}{k}
    \sum_{i=l+1}^{l+k}
    \left(
      \sum_y p_i^y - \max_y p_i^y
    \right),
  \end{equation}
  where the \emph{fuzziness} for a test object $x_i$
  is defined as the sum of all p-values apart from a largest one,
  i.e., as $\sum_y p_i^y - \max_y p_i^y$;
  smaller values of \eqref{eq:F} are preferable.
  If two conformal transducers lead to the same average fuzziness,
  the criterion compares the average credibilities \eqref{eq:credibility},
  with smaller values preferable.
\end{itemize}
Their counterparts depending on the significance level are:
\begin{itemize}
\item
  The \emph{M criterion} uses the percentage of objects $x_i$ in the test sequence
  for which the prediction set $\Gamma_i^{\epsilon}$ at significance level $\epsilon$ is \emph{multiple},
  i.e., contains more than one label.
  Smaller values are preferable.
  As a formula, the criterion prefers smaller
  \begin{equation}\label{eq:M}
    \frac{1}{k}
    \sum_{i=l+1}^{l+k}
    \III_{\{\left|\Gamma_i^{\epsilon}\right|>1\}},
  \end{equation}
  where $\III_E$ denotes the indicator function of the event $E$
  (taking value 1 if $E$ happens and 0 if not).
  When the percentage \eqref{eq:M} of multiple predictions is the same for two conformal predictors
  (which is a common situation: the percentage can well be zero
  when the data is clean and $\epsilon$ is not too demanding),
  the M criterion compares the percentages
  \begin{equation}\label{eq:empty}
    \frac{1}{k}
    \sum_{i=l+1}^{l+k}
    \III_{\{\Gamma_i^{\epsilon}=\emptyset\}}
  \end{equation}
  of empty predictions
  (larger values are preferable).
  This is a widely used criterion;
  in particular, it was used in \cite{Vovk/etal:2005book} and papers preceding it.
\item
  The \emph{E criterion}
  (where ``E'' stands for ``excess'')
  uses the average (over the test sequence, as usual) amount
  the size of the prediction set exceeds 1.
  In other words, the criterion gives the average number of excess labels
  in the prediction sets as compared with the ideal situation of one-element prediction sets.
  Smaller values are preferable for this criterion.
  As a formula, the criterion prefers smaller
  \begin{equation*}
    \frac{1}{k}
    \sum_{i=l+1}^{l+k}
    \left(
      \left|\Gamma_i^{\epsilon}\right| - 1
    \right)^+,
  \end{equation*}
  where $t^+:=\max(t,0)$.
  When these averages coincide for two conformal predictors,
  we compare the percentages \eqref{eq:empty} of empty predictions;
  larger values are preferable.
\end{itemize}
A criterion that is very similar to the M and E criteria is used by Lei in \cite{Lei:2014}
(Section~2.2);
that paper considers the binary case,
in which the difference between the M and E criteria disappears.
The difference of the criterion used in \cite{Lei:2014}
is that it prohibits empty predictions
(an intermediate approach would be to prefer smaller values for the number \eqref{eq:empty} of empty predictions).
Lei's criterion is extended to the multi-class case in \cite{Sadinle/etal:2016},
which proposes a modification of the E criterion
with a different treatment of empty predictions.

\subsection{Observed criteria}

The prior criteria discussed in the previous subsection
treat the largest p-value, or prediction sets of size 1,
in a special way.
The corresponding criteria of this subsection
attempt to achieve the same goal by using the observed label.

These are the observed counterparts of the non-basic prior $\epsilon$-free criteria:
\begin{itemize}
\item
  The \emph{OU} (``observed unconfidence'') \emph{criterion}
  uses the average observed unconfidence
  \[\frac{1}{k}\sum_{i=l+1}^{l+k}\max_{y\ne y_i}p^y_i\]
  over the test sequence,
  where the \emph{observed unconfidence} for a test example $(x_i,y_i)$
  is the largest p-value $p_i^y$ for the \emph{false labels}
  $y\ne y_i$.
  Smaller values are preferable for this test.
\item
  The \emph{OF} (``observed fuzziness'') \emph{criterion}
  uses the average sum of the p-values for the false labels,
  i.e.,
  \begin{equation}\label{eq:OF}
    \frac{1}{k}\sum_{i=l+1}^{l+k}\sum_{y\ne y_i}p^y_i;
  \end{equation}
  smaller values are preferable.
\end{itemize}
The counterparts of the last group depending on the significance level $\epsilon$ are:
\begin{itemize}
\item
  The \emph{OM criterion} uses the percentage of observed multiple predictions
  \begin{equation*}
    \frac{1}{k}
    \sum_{i=l+1}^{l+k}
    \III_{\{\Gamma_i^{\epsilon}\setminus\{y_i\}\ne\emptyset\}}
  \end{equation*}
  in the test sequence,
  where an \emph{observed multiple} prediction is defined to be a prediction set
  including a false label.
  Smaller values are preferable.
\item
  The \emph{OE criterion}
  (OE standing for ``observed excess'')
  uses the average number
  \begin{equation*}
    \frac{1}{k}
    \sum_{i=l+1}^{l+k}
    \left|
      \Gamma_i^{\epsilon} \setminus \{y_i\}
    \right|
  \end{equation*}
  of false labels included in the prediction sets at significance level $\epsilon$;
  smaller values are preferable.
\end{itemize}

The ten criteria used in this paper are given in Table~\ref{tab:criteria}.
Half of the criteria depend on the significance level $\epsilon$,
and the other half are the respective $\epsilon$-free versions.

\begin{table}[tb]
\caption{The ten criteria studied in this paper:
  the two basic ones in the upper section;
  the four other prior ones in the middle section;
  and the four observed ones in the lower section}
  \label{tab:criteria}

  \medskip
\begin{center}
\begin{tabular}{c|c}
  \hline
  \textbf{$\epsilon$-free} & \textbf{$\epsilon$-dependent}\\
  \hline\hline
  \emph{S (sum of p-values)} & \emph{N (number of labels)}\\
  \hline
  U (unconfidence) & M (multiple)\\
  F (fuzziness) & E (excess)\\
  \hline
  OU (observed unconfidence) & OM (observed multiple)\\
  \emph{OF (observed fuzziness)} & \emph{OE (observed excess)}\\
  \hline
\end{tabular}
\end{center}
\end{table}

In the case of binary classification problems, $\left|\mathbf{Y}\right|=2$,
the number of different criteria of efficiency
in Table~\ref{tab:criteria} reduces to six:
the criteria not separated by a vertical or horizontal line
(namely, U and F, OU and OF, M and E, and OM and OE)
coincide.

\section{Idealised Setting}
\label{sec:optimal}

Starting from this section we consider the limiting case of infinitely long training and test sequences
(and we will return to the realistic finitary case only in Section~\ref{sec:empirical},
where we describe our empirical studies).
To formalise the intuition of an infinitely long training sequence,
we assume that the prediction algorithm
is directly given the data-generating probability distribution $Q$ on $\mathbf{Z}$
instead of being given a training sequence.
Instead of conformity measures we will use \emph{idealised conformity measures}:
functions $A(Q,z)$ of $Q\in\mathcal{P}(\mathbf{Z})$
(where $\mathcal{P}(\mathbf{Z})$ is the set of all probability measures on $\mathbf{Z}$)
and $z\in\mathbf{Z}$.
We will fix the data-generating distribution $Q$ for the rest of the paper,
and so write the corresponding conformity scores as $A(z)$.
The \emph{idealised conformal predictor} corresponding to $A$
outputs the following prediction set $\Gamma^{\epsilon}(x)$
for each object $x\in\mathbf{X}$ and each significance level $\epsilon\in(0,1)$.
For each potential label $y\in\mathbf{Y}$ for $x$
define the corresponding \emph{p-value} as
\begin{multline}\label{eq:p-value}
  p^y
  =
  p(x,y)
  =
  p_A(x,y)
  =
  p_A(x,y,\tau)
  :=
  Q\{z\in\mathbf{Z}\mid A(z)<A(x,y)\}\\
  +
  \tau
  Q\{z\in\mathbf{Z}\mid A(z)=A(x,y)\}
\end{multline}
(it would have been more correct to write $A((x,y))$ and $Q(\{\ldots\})$,
but we often omit pairs of parentheses when there is no danger of ambiguity),
where $\tau$ is a random number distributed uniformly on $[0,1]$.
(The same random number $\tau$ is used in~\eqref{eq:p-value} for all $(x,y)$.)
The prediction set is
\begin{equation}\label{eq:prediction-set}
  \Gamma^{\epsilon}(x)
  =
  \Gamma_A^{\epsilon}(x)
  =
  \Gamma_A^{\epsilon}(x,\tau)
  :=
  \left\{
    y\in\mathbf{Y}
    \mid
    p(x,y)>\epsilon
  \right\}.
\end{equation}
The \emph{idealised conformal transducer} corresponding to $A$
outputs for each object $x\in\mathbf{X}$
the system of p-values $(p^y\mid y\in\mathbf{Y})$ defined by \eqref{eq:p-value};
in the idealised case we will usually use the alternative notation $p(x,y)$ for $p^y$.

We could have used the \emph{idealised conformity order} when defining the p-values \eqref{eq:p-value}:
$z\preceq z'$ is defined to mean $A(z)\le A(z')$.
Let us say that two idealised conformity measures are \emph{equivalent}
if they lead to the same idealised conformity order;
in other words, $A$ and $B$ are equivalent if, for all $z,z'\in\mathbf{Z}$,
$A(z)\le A(z')\Leftrightarrow B(z)\le B(z')$.

The standard properties of validity for conformal transducers and predictors
mentioned in the previous section
simplify in this idealised case as follows:
\begin{itemize}
\item
  If $(x,y)$ is generated from $Q$ and $\tau\in[0,1]$ is generated from the uniform distribution
  independently of $(x,y)$,
  $p(x,y)$
  is distributed uniformly on $[0,1]$.
\item
  Therefore, at each significance level $\epsilon$
  the idealised conformal predictor makes an error
  with probability~$\epsilon$.
\end{itemize}

The test sequence being infinitely long is formalised
by replacing the use of a test sequence in the criteria of efficiency
by averaging with respect to the data-generating probability distribution $Q$.
In the case of the top two and bottom two criteria in Table~\ref{tab:criteria}
(the ones set in italics)
this is done as follows.
An idealised conformity measure~$A$ is:
\begin{itemize}
\item
  \emph{S-optimal} if,
  for any idealised conformity measure $B$,
  \begin{equation}\label{eq:S}
    \Expect_{x,\tau}
    \sum_{y\in\mathbf{Y}}p_A(x,y)
    \le
    \Expect_{x,\tau}
    \sum_{y\in\mathbf{Y}}p_B(x,y),
  \end{equation}
  where the notation $\Expect_{x,\tau}$
  refers to the expected value when $x$ and $\tau$ are independent,
  $x\sim Q_{\mathbf{X}}$,
  and $\tau\sim U$;
  $Q_{\mathbf{X}}$ is the marginal distribution of $Q$ on $\mathbf{X}$,
  and $U$ is the uniform distribution on $[0,1]$;
\item
  \emph{N-optimal} if,
  for any idealised conformity measure~$B$
  and any significance level~$\epsilon$,
  \begin{equation*}
    \Expect_{x,\tau}\left|\Gamma^{\epsilon}_A(x)\right|
    \le
    \Expect_{x,\tau}\left|\Gamma^{\epsilon}_B(x)\right|;
  \end{equation*}
\item
  \emph{OF-optimal} if,
  for any idealised conformity measure~$B$,
  \begin{equation*}
    \Expect_{(x,y),\tau}
    \sum_{y'\ne y}p_A(x,y')
    \le
    \Expect_{(x,y),\tau}
    \sum_{y'\ne y}p_B(x,y'),
  \end{equation*}
  where the lower index $(x,y)$ in $\Expect_{(x,y),\tau}$
  refers to averaging over $(x,y)\sim Q$
  (with $(x,y)$ and $\tau$ independent);
\item
  \emph{OE-optimal} if,
  for any idealised conformity measure~$B$
  and any significance level~$\epsilon$,
  \[
    \Expect_{(x,y),\tau}
    \left|\Gamma^{\epsilon}_A(x)\setminus\{y\}\right|
    \le
    \Expect_{(x,y),\tau}
    \left|\Gamma^{\epsilon}_B(x)\setminus\{y\}\right|.
  \]
\end{itemize}
We will define the idealised versions of the other six criteria
listed in Table~\ref{tab:criteria} in Section~\ref{sec:not-probabilistic}.

\section{Probabilistic Criteria of Efficiency}

Our goal in this section is to characterise the optimal idealised conformity measures
for the four criteria of efficiency that are set in italics in Table~\ref{tab:criteria}.
We will assume in the rest of the paper that the set $\mathbf{X}$ is finite
(from the practical point of view, this is not a restriction);
since we consider the case of classification, $\left|\mathbf{Y}\right|<\infty$,
this implies that the whole example space $\mathbf{Z}$ is finite.
Without loss of generality, we also assume that the data-generating probability distribution $Q$
satisfies $Q_{\mathbf{X}}(x)>0$ for all $x\in\mathbf{X}$
(we often omit curly braces in expressions such as $Q_{\mathbf{X}}(\{x\})$):
we can always omit the $x$s for which $Q_{\mathbf{X}}(x)=0$.

The \emph{conditional probability (CP) idealised conformity measure} is
\begin{equation}\label{eq:CP}
  A(x,y)
  =
  Q(y\mid x)
  =
  Q_{\mathbf{Y}\mid\mathbf{X}}(y\mid x)
  :=
  \frac{Q(x,y)}{Q_{\mathbf{X}}(x)}.
\end{equation}
(In this paper, we will invariably use the shorter notation $Q(y\mid x)$
instead of the more precise $Q_{\mathbf{Y}\mid\mathbf{X}}(y\mid x)$;
we will never need $Q_{\mathbf{X}\mid\mathbf{Y}}$,
which could be defined analogously.)
This idealised conformity measure was introduced
by an anonymous referee of the conference version of \cite{\OCMIX},
but its non-idealised analogue in the case of regression had been used
in \cite{Lei/Wasserman:2013}
(following \cite{Lei/etal:2013} and literature on minimum volume prediction).
We say that an idealised conformity measure $A$ is a \emph{refinement}
of an idealised conformity measure $B$
if
\begin{equation}\label{eq:refinement}
  B(z_1)<B(z_2)
  \Longrightarrow
  A(z_1)<A(z_2)
\end{equation}
for all $z_1,z_2\in\mathbf{Z}$.
Let $\mathcal{R}(\CP)$ be the set of all refinements of the CP idealised conformity measure.
If $C$ is a criterion of efficiency (one of the ten criteria in Table~\ref{tab:criteria}),
we let $\mathcal{O}(C)$ stand for the set of all $C$-optimal idealised conformity measures.

\begin{theorem}\label{thm:CP}
  $\mathcal{O}(\SSS)=\mathcal{O}(\OF)=\mathcal{O}(\NNN)=\mathcal{O}(\OE)=\mathcal{R}(\CP)$.
\end{theorem}

We say that an efficiency criterion is \emph{probabilistic}
if the CP idealised conformity measure is always optimal for it.
We will also use two modifications of this definition:
an efficiency criterion is \emph{strongly probabilistic}
if any refinement of the CP idealised conformity measure is optimal for it,
and it is \emph{weakly probabilistic}
if some refinement of the CP idealised conformity measure is optimal for it.
We will say that it is \emph{BW probabilistic}
(or \emph{binary-weakly probabilistic})
if some refinement of the CP idealised conformity measure is optimal for it
whenever $\left|\mathbf{Y}\right|=2$.
Theorem~\ref{thm:CP} shows that four of our ten criteria are strongly probabilistic,
namely S, N, OF, and OE
(they are set in italics in Table~\ref{tab:criteria}).
In the next section we will see that in general the other six criteria
are not probabilistic (they are only BW probabilistic).
The intuition behind probabilistic criteria
will be briefly discussed also in the next section.

\begin{proof}[Proof of Theorem~\ref{thm:CP}]
  We start from proving $\mathcal{R}(\CP)=\mathcal{O}(\NNN)$.
  Let $A$ be any idealised conformity measure.
  Fix for a moment a significance level $\epsilon$.
  For each example $(x,y)\in\mathbf{Z}$,
  let $P(x,y)$ be the probability that the idealised conformal predictor based on $A$ makes an error
  on the example $(x,y)$ at the significance level $\epsilon$,
  i.e., the probability (over $\tau$) of $y\notin\Gamma^{\epsilon}_A(x)$.
  It is clear from \eqref{eq:p-value} and \eqref{eq:prediction-set}
  that $P$ takes at most three possible values (0, 1, and an intermediate value)
  and that
  \begin{equation}\label{eq:constraint}
    \sum_{x,y}
    Q(x,y)P(x,y)
    =
    \epsilon
  \end{equation}
  (which just reflects the fact that the probability of error is $\epsilon$).
  Vice versa, any $P$ satisfying these properties
  will also satisfy
  \[
    \forall(x,y):
    P(x,y)
    =
    \Prob_{\tau}
    \left(
      y\notin\Gamma^{\epsilon}_A(x,\tau)
    \right)
  \]
  for some $A$,
  $\Prob_{\tau}$ standing for the probability when $\tau\sim U$.
  Let us see when we will have $A\in\mathcal{O}(\NNN)$ ($A$ is an N-optimal idealised conformity measure).
  Define $Q'$ to be the probability measure on $\mathbf{Z}$ such that $Q'_{\mathbf{X}}=Q_{\mathbf{X}}$
  and $Q'(y\mid x)=1/\left|\mathbf{Y}\right|$ does not depend on $y$.
  The N criterion at significance level $\epsilon$ for $A$
  can be evaluated as
  \begin{equation}\label{eq:objective}
    \Expect_{x,\tau}
    \left|
      \Gamma^{\epsilon}_A(x)
    \right|
    =
    \left|\mathbf{Y}\right|
    \left(
      1
      -
      \sum_{(x,y)\in\mathbf{Z}}
      Q'(x,y)P(x,y)
    \right);
  \end{equation}
  this expression should be minimised,
  i.e., $\sum_{(x,y)}Q'(x,y)P(x,y)$ should be maximised,
  under the restriction \eqref{eq:constraint}.
  Let us apply the Neyman--Pearson fundamental lemma (\cite{Lehmann:1986}, Sect.~3.2, Theorem~1)
  using $Q$ as the null and $Q'$ as the alternative hypotheses.
  We can see that $\Expect_{x,\tau}\left|\Gamma^{\epsilon}_A(x)\right|$
  takes its minimal value if and only if there exist thresholds $k_1=k_1(\epsilon)$, $k_2=k_2(\epsilon)$, and $k_3=k_3(\epsilon)$
  such that:
  \begin{itemize}
  \item
    $Q\{(x,y)\mid Q(y\mid x)<k_1\}<\epsilon\le Q\{(x,y)\mid Q(y\mid x)\le k_1\}$,
  \item
    $k_2<k_3$,
  \item
    $A(x,y)<k_2$ if $Q(y\mid x)<k_1$,
  \item
    $k_2<A(x,y)<k_3$ if $Q(y\mid x)=k_1$,
  \item
    $A(x,y)>k_3$ if $Q(y\mid x)>k_1$.
  \end{itemize}
  This will be true for all $\epsilon$ if and only if $Q(y\mid x)$ is a function of $A(x,y)$
  (meaning that there exists a function $F$ such that, for all $(x,y)$, $Q(y\mid x)=F(A(x,y))$).
  This completes the proof of $\mathcal{R}(\CP)=\mathcal{O}(\NNN)$.

  Next we show that $\mathcal{O}(\NNN)=\mathcal{O}(\SSS)$.
  The chain of equalities
  \begin{multline}\label{eq:basic}
    \sum_{y \in \mathbf{Y}}
    p(x,y)
    =
    \sum_{y \in \mathbf{Y}}
    \int_0^1
    \III_{\{p(x,y)>\epsilon\}}
    \dd{\epsilon}\\
    =
    \int_0^1
      \sum_{y \in \mathbf{Y}}
      \III_{\{p(x,y)>\epsilon\}}
    \dd{\epsilon}
    =
    \int_0^1
      \left|\Gamma^{\epsilon}(x)\right|
    \dd{\epsilon}
  \end{multline}
  (which will be used as the model in several other proofs in the rest of this paper)
  implies, by Fubini's theorem,
  \begin{equation}\label{eq:expectation}
    \Expect_{x,\tau}
    \sum_{y \in \mathbf{Y}}
    p(x,y)
    =
    \int_0^1
    \Expect_{x,\tau}
    \left|\Gamma^{\epsilon}(x)\right|
    \dd{\epsilon}.
  \end{equation}
  We can see that $A\in\mathcal{O}(\SSS)$ whenever $A\in\mathcal{O}(\NNN)$:
  indeed, any N-optimal idealised conformity measure minimises the expectation
  $\Expect_{x,\tau}\left|\Gamma^{\epsilon}(x)\right|$
  on the right-hand side of \eqref{eq:expectation}
  for all $\epsilon$ simultaneously,
  and so minimises the whole right-hand-side,
  and so minimises the left-hand-side.
  On the other hand,
  $A\notin\mathcal{O}(\SSS)$ whenever $A\notin\mathcal{O}(\NNN)$:
  indeed, if an idealised conformity measure fails
  to minimise the expectation $\Expect_{x,\tau}\left|\Gamma^{\epsilon}(x)\right|$
  on the right-hand side of \eqref{eq:expectation}
  for some $\epsilon$,
  it fails to do so for all $\epsilon$ in a non-empty open interval
  (because of the right-continuity of $\Expect_{x,\tau}\left|\Gamma^{\epsilon}(x)\right|$ in $\epsilon$,
  which is proved in Lemma~\ref{lem:right-continuous}(b) below),
  and therefore, it does not minimise the right-hand side of \eqref{eq:expectation}
  (any N-optimal idealised conformity measure, such as the CP idealised conformity measure,
  will give a smaller value),
  and therefore, it does not minimise the left-hand side of~\eqref{eq:expectation}.

  The equality $\mathcal{O}(\SSS)=\mathcal{O}(\OF)$ follows from
  \[
    \Expect_{x,\tau} \sum_{y} p(x,y)
    =
    \Expect_{(x,y),\tau} \sum_{y'\ne y}p(x,y') + \frac12,
  \]
  where we have used the fact that $p(x,y)$ is distributed uniformly on $[0,1]$
  when $((x,y),\tau)\sim Q\times U$
  (see \cite{Vovk/etal:2005book}).

  Finally, we notice that $\mathcal{O}(\NNN)=\mathcal{O}(\OE)$.
  Indeed, for any significance level $\epsilon$,
  \[
    \Expect_{x,\tau} |\Gamma^\epsilon(x)|
    =
    \Expect_{(x,y),\tau} |\Gamma^\epsilon(x) \setminus \{y\}|
    +
    (1-\epsilon),
  \]
  again using the fact that $p(x,y)$ is distributed uniformly on $[0,1]$
  and so $\Prob_{(x,y),\tau}(y\in\Gamma^\epsilon(x)) = 1 - \epsilon$,
  where $\Prob_{(x,y),\tau}$ refers to the probability when $(x,y)\sim Q$ and $\tau\sim U$ are independent.
\end{proof}

The following lemma was used in the proof of Theorem~\ref{thm:CP}.
\begin{lemma}\label{lem:right-continuous}
  (a) The function $\Gamma^{\epsilon}(x)=\Gamma^{\epsilon}(x,\tau)$
      of $\epsilon$ is right-continuous for fixed $x$ and $\tau$.
  (b) The function $\Expect_{x,\tau}\left|\Gamma^{\epsilon}(x)\right|$
      is right-continuous in $\epsilon$.
\end{lemma}
\begin{proof}
  Let us first check (a).
  We have
  (i) $p(x,y,\tau)>\epsilon$ for all $y\in\Gamma^{\epsilon}(x,\tau)$, and
  (ii) $p(x,y,\tau)\le\epsilon$ for all $y\notin\Gamma^{\epsilon}(x,\tau)$.
  If we increase $\epsilon$,
  (ii) will be still satisfied,
  and if the increase is sufficiently small,
  (i) will be also satisfied and, therefore,
  $\Gamma^{\epsilon}(x,\tau)$ will not change.
  As for (b),
  the right-continuity of $\Gamma^{\epsilon}(x,\tau)$ in $\epsilon$
  implies the right-continuity of $\left|\Gamma^{\epsilon}(x,\tau)\right|$ in $\epsilon$,
  which implies the right-continuity of
  $\Expect_{x,\tau}\left|\Gamma^{\epsilon}(x,\tau)\right|$ in $\epsilon$
  by the Lebesgue dominated convergence theorem.
\end{proof}

\begin{remark}
  The statement $\mathcal{O}(\SSS)=\mathcal{R}(\CP)$ of Theorem~\ref{thm:CP}
  can be generalised to the criterion $\SSS_{\phi}$ preferring small values of
  \begin{equation*}
    \frac{1}{k}\sum_{i=l+1}^{l+k}\sum_{y}\phi(p^y_i)
    \text{ or }
    \Expect_{x,\tau}\sum_y\phi(p(x,y))
  \end{equation*}
  (instead of \eqref{eq:S-real} or \eqref{eq:S}, respectively),
  where $\phi:[0,1]\to\mathbb{R}$ is a fixed continuously differentiable strictly increasing function,
  not necessarily the identity function.
  Namely, we still have $\mathcal{O}(\SSS_{\phi})=\mathcal{R}(\CP)$.
  Indeed, we can assume, without loss of generality,
  that $\phi(0)=0$ and $\phi(1)=1$ and replace \eqref{eq:basic} by
  \begin{multline*}
    \sum_{y \in \mathbf{Y}}
    \phi(p(x,y))
    =
    \sum_{y \in \mathbf{Y}}
    \int_0^1
    \III_{\{\phi(p(x,y))>\epsilon\}}
    \dd{\epsilon}
    =
    \int_0^1
    \sum_{y \in \mathbf{Y}}
    \III_{\{p(x,y)>\phi^{-1}(\epsilon)\}}
    \dd{\epsilon}\\
    =
    \int_0^1
    \left|\Gamma^{\phi^{-1}(\epsilon)}(x)\right|
    \dd{\epsilon}
    =
    \int_0^1
    \left|\Gamma^{\epsilon'}(x)\right|
    \phi'(\epsilon')
    \dd{\epsilon'},
  \end{multline*}
  where $\phi'$ is the (continuous) derivative of $\phi$,
  and then use the same argument as before.
\end{remark}

\section{Criteria of Efficiency that are not Probabilistic}
\label{sec:not-probabilistic}

Now we define the idealised analogues of the six criteria
that are not set in italics in Table~\ref{tab:criteria}.
An idealised conformity measure~$A$ is:
\begin{itemize}
\item
  \emph{U-optimal} if,
  for any idealised conformity measure $B$,
  we have either
  \begin{equation}\label{eq:U-optimal-1}
    \Expect_{x,\tau}\min_y\max_{y'\ne y}p_A(x,y')
    <
    \Expect_{x,\tau}\min_y\max_{y'\ne y}p_B(x,y')
  \end{equation}
  or both
  \begin{equation}\label{eq:U-optimal-2}
    \Expect_{x,\tau}\min_y\max_{y'\ne y}p_A(x,y')
    =
    \Expect_{x,\tau}\min_y\max_{y'\ne y}p_B(x,y')
  \end{equation}
  and
  \begin{equation}\label{eq:UF-optimal-3}
    \Expect_{x,\tau}\max_{y}p_A(x,y)
    \le
    \Expect_{x,\tau}\max_{y}p_B(x,y);
  \end{equation}
\item
  \emph{M-optimal} if,
  for any idealised conformity measure~$B$
  and any significance level~$\epsilon$,
  we have either
  \begin{equation}\label{eq:M-optimal-1}
    \Prob_{x,\tau}(\left|\Gamma^{\epsilon}_A(x)\right|>1)
    <
    \Prob_{x,\tau}(\left|\Gamma^{\epsilon}_B(x)\right|>1)
  \end{equation}
  or both
  \begin{equation}\label{eq:M-optimal-2}
    \Prob_{x,\tau}(\left|\Gamma^{\epsilon}_A(x)\right|>1)
    =
    \Prob_{x,\tau}(\left|\Gamma^{\epsilon}_B(x)\right|>1)
  \end{equation}
  and
  \begin{equation}\label{eq:ME-optimal-3}
    \Prob_{x,\tau}(\left|\Gamma^{\epsilon}_A(x)\right|=0)
    \ge
    \Prob_{x,\tau}(\left|\Gamma^{\epsilon}_B(x)\right|=0);
  \end{equation}
\item
  \emph{F-optimal} if,
  for any idealised conformity measure $B$,
  we have either
  \begin{equation}\label{eq:F-optimal-1}
    \Expect_{x,\tau}
    \Bigl(
       \sum_y p_A(x,y)
       -
       \max_y p_A(x,y)
    \Bigr)
    <
    \Expect_{x,\tau}
    \Bigl(
       \sum_y p_B(x,y)
       -
       \max_y p_B(x,y)
    \Bigr)
  \end{equation}
  or both
  \begin{equation}\label{eq:F-optimal-2}
    \Expect_{x,\tau}
    \Bigl(
       \sum_y p_A(x,y)
       -
       \max_y p_A(x,y)
    \Bigr)
    =
    \Expect_{x,\tau}
    \Bigl(
       \sum_y p_B(x,y)
       -
       \max_y p_B(x,y)
    \Bigr)
  \end{equation}
  and \eqref{eq:UF-optimal-3};
\item
  \emph{E-optimal} if,
  for any idealised conformity measure~$B$
  and any significance level~$\epsilon$,
  we have either
  \begin{equation}\label{eq:E-optimal-1}
    \Expect_{x,\tau}
    \bigl(
      \left(
        \left|\Gamma^{\epsilon}_A(x)\right|-1
      \right)^+
    \bigr)
    <
    \Expect_{x,\tau}
    \bigl(
      \left(
        \left|\Gamma^{\epsilon}_B(x)\right|-1
      \right)^+
    \bigr)
  \end{equation}
  or both
  \begin{equation}\label{eq:E-optimal-2}
    \Expect_{x,\tau}
    \bigl(
      \left(
        \left|\Gamma^{\epsilon}_A(x)\right|-1
      \right)^+
    \bigr)
    =
    \Expect_{x,\tau}
    \bigl(
      \left(
        \left|\Gamma^{\epsilon}_B(x)\right|-1
      \right)^+
    \bigr)
  \end{equation}
  and \eqref{eq:ME-optimal-3};
\item
  \emph{OU-optimal} if,
  for any idealised conformity measure~$B$,
  \begin{equation}\label{eq:OU-optimal}
    \Expect_{(x,y),\tau}\max_{y'\ne y}p_A(x,y')
    \le
    \Expect_{(x,y),\tau}\max_{y'\ne y}p_B(x,y');
  \end{equation}
\item
  \emph{OM-optimal} if,
  for any idealised conformity measure~$B$
  and any significance level~$\epsilon$,
  \begin{equation}\label{eq:OM-optimal}
    \Prob_{(x,y),\tau}(\Gamma^{\epsilon}_A(x)\setminus\{y\}\ne\emptyset) % \\
    \le
    \Prob_{(x,y),\tau}(\Gamma^{\epsilon}_B(x)\setminus\{y\}\ne\emptyset).
  \end{equation}
\end{itemize}
In the following three definitions we follow \cite{Vovk/etal:2005book}, Chapter~3.
The \emph{predictability} of $x\in\mathbf{X}$ is
\begin{equation}\label{eq:f}
  f(x)
  :=
  \max_{y\in\mathbf{Y}}
  Q(y\mid x).
\end{equation}
A \emph{choice function} $\hat y:\mathbf{X}\to\mathbf{Y}$
is defined by the condition
\begin{equation}\label{eq:choice}
  \forall x\in\mathbf{X}:
  f(x) = Q(\hat y(x)\mid x).
\end{equation}
Define the \emph{signed predictability idealised conformity measure} corresponding to $\hat y$ by
\[
  A(x,y)
  :=
  \begin{cases}
    f(x) & \text{if $y=\hat y(x)$}\\
    -f(x) & \text{if not};
  \end{cases}
\]
a \emph{signed predictability (SP) idealised conformity measure}
is the signed predictability idealised conformity measure corresponding to some choice function.

For the following two theorems we will need to modify the notion of refinement.
Let $\mathcal{R}'(\SP)$ be the set of all idealised conformity measures $A$
such that there exists an SP idealised conformity measure $B$
that satisfies both \eqref{eq:refinement} and
\begin{equation*}
  B(x,y_1)=B(x,y_2)
  \Longrightarrow
  A(x,y_1)=A(x,y_2)
\end{equation*}
for all $x\in\mathbf{X}$ and $y_1,y_2\in\mathbf{Y}$.

\begin{theorem}\label{thm:SP}
  $\mathcal{O}(\UUU)=\mathcal{O}(\MMM)=\mathcal{R}'(\SP)$.
\end{theorem}

Theorems~\ref{thm:SP}--\ref{thm:MSP} will be proved in Section~\ref{sec:proofs} below.

Define the \emph{MCP (modified conditional probability) idealised conformity measure}
corresponding to a choice function $\hat y$ by
\[
  A(x,y)
  :=
  \begin{cases}
    Q(y\mid x) & \text{if $y=\hat y(x)$}\\
    Q(y\mid x)-1 & \text{if not};
  \end{cases}
\]
an \emph{MCP idealised conformity measure} is an idealised conformity measure
corresponding to some choice function;
$\mathcal{R}(\MCP)$ is defined analogously to $\mathcal{R}(\CP)$
but using MCP idealised conformity measures rather than the CP idealised conformity measure.

\begin{theorem}\label{thm:MCP}
  $\mathcal{O}(\FFF)=\mathcal{O}(\EEE)=\mathcal{R}(\MCP)$.
\end{theorem}

Of course, Theorems~\ref{thm:SP} and~\ref{thm:MCP} are equivalent
when $\left|\mathbf{Y}\right|=2$.

The \emph{modified signed predictability (MSP) idealised conformity measure}
is defined by
\[
  A(x,y)
  :=
  \begin{cases}
    f(x) & \text{if $f(x)>1/2$ and $y=\hat y(x)$}\\
    0 & \text{if $f(x)\le1/2$}\\
    -f(x) & \text{if $f(x)>1/2$ and $y\ne\hat y(x)$},
  \end{cases}
\]
where $f$ is the predictability function \eqref{eq:f};
notice that this definition is unaffected by the choice of the choice function.
Let $\mathcal{R}''(\MSP)$ be the set of all refinements $A$ of the MSP idealised conformity measure
such that, for all $x\in\mathbf{X}$ and all $y_1,y_2\in\mathbf{Y}$:
\begin{align*}
  f(x)\ge0.5 \And Q(y_1\mid x)<0.5 \And Q(y_2\mid x)<0.5 &\Longrightarrow A(x,y_1)=A(x,y_2)\\
  f(x)<0.5 &\Longrightarrow A(x,y_1)=A(x,y_2).
\end{align*}

\begin{theorem}\label{thm:MSP}
  $\mathcal{O}(\OU)=\mathcal{O}(\OM)=\mathcal{R}''(\MSP)$.
\end{theorem}

Table~\ref{tab:results} summarises the results given above.
For each of the criteria listed in Table~\ref{tab:criteria}
it gives an optimal idealised conformity measure
and cites the result asserting the optimality of that idealised conformity measure.

\begin{table}[tb]
  \caption{Idealised conformity measures that are optimal for the ten criteria of efficiency
    given in Table~\ref{tab:criteria};
    the arrangement of the criteria is the same as in Table~\ref{tab:criteria}}
    \label{tab:results}

  \medskip
  \begin{center}
  \begin{tabular}{c|c}
    \hline
    \textbf{$\epsilon$-free} & \textbf{$\epsilon$-dependent}\\
    \hline\hline
    \emph{S: CP (Theorem~\ref{thm:CP})} & \emph{N: CP (Theorem~\ref{thm:CP})}\\
    \hline
    U: SP (Theorem~\ref{thm:SP}) & M: SP (Theorem~\ref{thm:SP})\\
    F: MCP (Theorem~\ref{thm:MCP}) & E: MCP (Theorem~\ref{thm:MCP})\\
    \hline
    OU: MSP (Theorem~\ref{thm:MSP}) & OM: MSP (Theorem~\ref{thm:MSP})\\
    \emph{OF: CP (Theorem~\ref{thm:CP})} & \emph{OE: CP (Theorem~\ref{thm:CP})}\\
    \hline
  \end{tabular}
  \end{center}
\end{table}

Theorems~\ref{thm:SP}--\ref{thm:MSP} show that the six criteria that are not set in italics
in Table~\ref{tab:criteria} are not probabilistic
(however, we will see in Corollary~\ref{cor:binary-probabilistic} below that they are BW probabilistic).
These are simple explicit examples
(inevitably involving label spaces $\mathbf{Y}$ with $\left|\mathbf{Y}\right|>2$)
showing that they are not even weakly probabilistic:
\begin{itemize}
\item
  Let $\mathbf{X}=\{1\}$, $\mathbf{Y}=\{1,2,3\}$, and
  \begin{equation}\label{eq:U-M-example}
    \begin{aligned}
      Q_{\mathbf{X}}(1)&=1 & Q(1\mid 1)&=0.2 & Q(2\mid 1)&=0.3 & Q(3\mid 1)&=0.5.
    \end{aligned}
  \end{equation}
  (Remember that, in this paper, $Q(y\mid x)$ always means $Q_{\mathbf{Y}\mid\mathbf{X}}(y\mid x)$.)
  In this case, all refinements of the CP idealised conformity measure are equivalent.
  The U criterion is not probabilistic since the expression
  \begin{equation}\label{eq:U-expression}
    \Expect_{x,\tau}\min_y\max_{y'\ne y}p(x,y')
  \end{equation}
  (cf.\ \eqref{eq:U-optimal-1})
  is $0.35$ for the CP idealised conformity measure
  and is smaller, $0.25$, for the SP idealised conformity measure.
  The M criterion is not probabilistic since at significance level $\epsilon=0.2$
  the CP idealised conformity measure gives the predictor
  $\Gamma^{\epsilon}(1)=\{2,3\}$ (a.s.),
  and so
  \begin{equation*}
    \Prob_{x,\tau}(\left|\Gamma^{\epsilon}_{\CP}(x)\right|>1)
    = 1 > 0.6 =
    \Prob_{x,\tau}(\left|\Gamma^{\epsilon}_{\SP}(x)\right|>1)
  \end{equation*}
  (cf.\ \eqref{eq:M-optimal-1}).
\item
  Let $\mathbf{X}=\{1,2\}$, $\mathbf{Y}=\{1,2,3\}$, and, for a small $\delta>0$,
  \begin{align*}
    Q_{\mathbf{X}}(1)&=0.5 & Q(1\mid 1)&=1/3-\delta & Q(2\mid 1)&=1/3 & Q(3\mid 1)&=1/3+\delta \\
    Q_{\mathbf{X}}(2)&=0.5 & Q(1\mid 2)&=1/3-5\delta & Q(2\mid 2)&=1/3+2\delta & Q(3\mid 2)&=1/3+3\delta.
  \end{align*}
  The CP idealised conformity measure again has only equivalent refinements.
  The F criterion is not probabilistic since the expression
  \begin{equation}\label{eq:F-expression}
    \Expect_{x,\tau}
    \Bigl(
       \sum_y p(x,y)
       -
       \max_y p(x,y)
    \Bigr)
  \end{equation}
  (cf.\ \eqref{eq:F-optimal-1})
  is $3/4+O(\delta)$ for the CP idealised conformity measure
  and is smaller (provided $\delta$ is sufficiently small), $2/3+O(\delta)$,
  for the MCP idealised conformity measure (which is unique).
  The E criterion is not probabilistic since at significance level $\epsilon=2/3$
  the CP idealised conformity measure has a larger expected excess (for small $\delta$)
  than the MCP idealised conformity measure (whose expected excess is zero):
  \begin{equation*}
    \Expect_{x,\tau}
    \bigl(
      \left(
        \left|\Gamma^{\epsilon}_{\CP}(x)\right|-1
      \right)^+
    \bigr)
    = 0.5 + O(\delta) > 0 =
    \Expect_{x,\tau}
    \bigl(
      \left(
        \left|\Gamma^{\epsilon}_{\MCP}(x)\right|-1
      \right)^+
    \bigr)
  \end{equation*}
  (cf.\ \eqref{eq:E-optimal-1}).
\item
  Let us again set $\mathbf{X}=\{1\}$ and $\mathbf{Y}=\{1,2,3\}$, and define $Q$ by \eqref{eq:U-M-example}.
  The OU criterion is not probabilistic since the expression
  \begin{equation}\label{eq:OU-expression}
    \Expect_{(x,y),\tau}
    \max_{y'\ne y}
    p(x,y')
  \end{equation}
  (cf.\ \eqref{eq:OU-optimal})
  is $0.55$ for the CP idealised conformity measure
  and is smaller, $0.5$, for the MSP idealised conformity measure.
  The OM criterion is not probabilistic since at significance level $\epsilon=0.2$
  the CP idealised conformity measure gives the predictor
  $\Gamma^{\epsilon}(1)=\{2,3\}$ (a.s.),
  and so
  \begin{equation*}
    \Prob_{(x,y),\tau}(\Gamma^{0.2}_{\CP}(x)\setminus\{y\}\ne\emptyset)
    = 1 > 0.8 =
    \Prob_{(x,y),\tau}(\Gamma^{0.2}_{\MSP}(x)\setminus\{y\}\ne\emptyset)
  \end{equation*}
  (cf.\ \eqref{eq:OM-optimal}).
\end{itemize}

\begin{corollary}\label{cor:binary-probabilistic}
  All ten criteria of efficiency in Table~\ref{tab:criteria} are BW probabilistic.
\end{corollary}

\begin{proof}
  Criteria S, N, OF, and OE are BW probabilistic
  by Theorem~\ref{thm:CP}.
  Criteria OU and OM are identical to OF and OE, respectively, in the binary case,
  and so are also BW probabilistic.
  Criteria F and E are identical to U and M, respectively, in the binary case,
  and so our task reduces to proving that U and M are BW probabilistic.
  By Theorem~\ref{thm:SP}, it suffices to check $\mathcal{R}(\CP)\cap\mathcal{R}'(\SP)\ne\emptyset$,
  which is obvious:
  SP is in both $\mathcal{R}(\CP)$ and $\mathcal{R}'(\SP)$ when $\left|\mathbf{Y}\right|=2$.
\end{proof}

Criteria of efficiency that are not probabilistic
are somewhat analogous to ``improper scoring rules''
in probability forecasting
(see, e.g., \cite{Dawid:ESS2006PF} and \cite{Gneiting/Raftery:2007}).
The optimal idealised conformity measures for the criteria of efficiency given in this paper
that are not probabilistic have clear disadvantages, such as:
\begin{itemize}
\item
  They depend on the arbitrary choice of a choice function.
  In many cases there is a unique choice function,
  but the possibility of non-uniqueness is still awkward.
\item
  They encourage ``strategic behaviour''
  (such as ignoring the differences, which may be very substantial,
  between potential labels other than $\hat y(x)$ for a test object $x$
  when using the M criterion in the case $\left|\mathbf{Y}\right|>2$).
\end{itemize}
However, we do not use the terminology ``proper/improper''
in the case of criteria of efficiency for conformal prediction
since it is conceivable that some non-probabilistic criteria of efficiency
may still turn out to be useful.

\section{Proofs of Theorems~\ref{thm:SP}--\ref{thm:MSP}}
\label{sec:proofs}

The proofs in this section will be slightly less formal than the proof of Theorem~\ref{thm:CP};
in particular, all references to the Neyman--Pearson lemma will be implicit.

\subsection{Proof of Theorem~\protect\ref{thm:SP}}

We start from checking that $\mathcal{O}(\MMM)=\mathcal{R}'(\SP)$
(essentially reproducing the argument given in the second parts of the proofs of Propositions~3.3 and~3.4
in \cite{Vovk/etal:2005book}).
We will analyze the requirements imposed by being M-optimal
on the prediction set $\Gamma^{\epsilon}$ starting from small values of $\epsilon\in(0,1)$.
(In this paper we only consider $\epsilon$ in the interval $(0,1)$,
even if this restriction is not mentioned explicitly.)

Let $f_1>f_2>\cdots>f_n>0$ be the list of the predictabilities (see \eqref{eq:f})
of all objects $x\in\mathbf{X}$,
with all duplicates removed and the remaining predictabilities sorted in the decreasing order.
It is clear that an M-optimal idealised conformity measure
will assign the lowest conformity to the group of examples $(x,y)$
with $f(x)=f_1$ and $y\ne\hat y(x)$ for some choice function $\hat y$
(see \eqref{eq:choice}).
The conformity of such examples can be different unless they contain the same object
(in which case it must be the same);
the conformity of any example in any other group must be higher
than the conformity of the examples in this first group.
If these conditions are satisfied for some idealised conformity measure $A$,
$A$ will satisfy \eqref{eq:M-optimal-1} or \eqref{eq:M-optimal-2}
for any idealised conformity measure $B$
and any
\[
  \epsilon
  \in
  \left(
    0,
    Q
    \left\{
      (x,y)
      \mid
      f(x)=f_1\And y\ne\hat y(x)
    \right\}
  \right].
\]
The second least conforming group of examples consists of $(x,y)$
with $f(x)=f_2$ and $y\ne\hat y(x)$ for some choice function $\hat y$.
The conformity of examples in the second group can again be different unless they contain the same object.
These and previous conditions
ensure that $A$ will satisfy \eqref{eq:M-optimal-1} or \eqref{eq:M-optimal-2}
for any
\[
  \epsilon
  \in
  \left(
    0,
    Q
    \left\{
      (x,y)
      \mid
      f(x)\ge f_2\And y\ne\hat y(x)
    \right\}
  \right].
\]
Continuing in such a way,
we will obtain a choice function $\hat y$ and the conformity ordering
for the examples whose label is not chosen by that choice function $\hat y$.
All these examples are divided into $n$ groups,
and each elements of the $i$th group is coming before each element of the $j$th group when $i<j$;
in the end we will get $2n$ groups satisfying this property.
The first $n$ groups take care of
\[
  \epsilon
  \in
  \left(
    0,
    Q
    \left\{
      (x,y)
      \mid
      y\ne\hat y(x)
    \right\}
  \right].
\]
The next, $(n+1)$th, group of examples are $(x,\hat y(x))\in\mathbf{Z}$ with $f(x)=f_n$;
they can be ordered in any way between themselves.
If the conditions listed so far are satisfied for an idealised conformity measure $A$,
$A$ will satisfy \eqref{eq:M-optimal-1}--\eqref{eq:ME-optimal-3}
for any idealised conformity measure $B$ and any
\[
  \epsilon
  \in
  \left(
    0,
    Q
    \left\{
      (x,y)
      \mid
      y\ne\hat y(x)
      \text{ or }
      \left(
        y=\hat y(x)\And f(x)=f_n
      \right)
    \right\}
  \right].
\]
The following, $(n+2)$th, group consists of $(x,\hat y(x))\in\mathbf{Z}$ with $f(x)=f_{n-1}$.
Continuing in the same way until all examples are exhausted,
we will obtain a refinement of the SP idealised conformity measure
that belongs to $\mathcal{R}'(\SP)$.

This proof of $\mathcal{O}(\MMM)=\mathcal{R}'(\SP)$ demonstrates the following property
of M-optimal idealised conformity measures.

\begin{corollary}
  If $A\in\mathcal{O}(\MMM)$,
  \[
    \Prob_{x,\tau}(\left|\Gamma^{\epsilon}_{A}(x)\right|>1)
    \Prob_{x,\tau}(\left|\Gamma^{\epsilon}_{A}(x)\right|=0)
    =
    0
  \]
  at each significance level $\epsilon$.
\end{corollary}

Let us now check that $\mathcal{O}(\UUU)=\mathcal{O}(\MMM)$.
Analogously to \eqref{eq:basic} and \eqref{eq:expectation},
we have, for a given idealised conformity measure $A$ (omitted from our notation),
\begin{multline}\label{eq:Ilia-SP-1}
  \Expect_{x,\tau}
  \min_y
  \max_{y'\ne y}
  p(x,y',\tau)
  =
  \Expect_{x,\tau}
  \int_0^1
  \III_
  {\{
    \min_y
    \max_{y'\ne y}
    p(x,y',\tau)
    >
    \epsilon
  \}}
  \dd{\epsilon}\\
  =
  \Expect_{x,\tau}
  \int_0^1
  \III_
  {\{
    \left|
      \Gamma^{\epsilon}(x)
    \right|
    >
    1
  \}}
  \dd{\epsilon}
  =
  \int_0^1
    \Prob_{x,\tau}
    \left(
      \left|
        \Gamma^{\epsilon}(x)
      \right|
      >
      1
    \right)
  \dd{\epsilon}.
\end{multline}
Similarly, we have
\begin{multline}\label{eq:Ilia-SP-2}
  \Expect_{x,\tau}
  \max_y
  p(x,y,\tau)
  =
  \Expect_{x,\tau}
  \int_0^1
  \III_
  {\{
    \max_y
    p(x,y,\tau)
    >
    \epsilon
  \}}
  \dd{\epsilon}\\
  =
  \Expect_{x,\tau}
  \int_0^1
  \III_
  {\{
    \left|
      \Gamma^{\epsilon}(x)
    \right|
    >
    0
  \}}
  \dd{\epsilon}
  =
  \int_0^1
    \Prob_{x,\tau}
    \left(
      \left|
        \Gamma^{\epsilon}(x)
      \right|
      >
      0
    \right)
  \dd{\epsilon}\\
  =
  1
  -
  \int_0^1
    \Prob_{x,\tau}
    \left(
      \left|
        \Gamma^{\epsilon}(x)
      \right|
      =
      0
    \right)
  \dd{\epsilon}.
\end{multline}

Our argument will also use the following continuity property for idealised conformal predictors.
(For now, we only need parts (a) and (b).)

\begin{corollary}\label{cor:right-continuous}
  The functions
  \begin{enumerate}[label=(\alph*)]
  \item % (a)
    $\Prob_{x,\tau}\left(\left|\Gamma^{\epsilon}(x)\right|>1\right)$
  \item % (b) The function
    $\Prob_{x,\tau}\left(\left|\Gamma^{\epsilon}(x)\right|=0\right)$
  \item % (c) The function
    $\Expect_{x,\tau}\left(\left(\left|\Gamma^{\epsilon}(x)\right|-1\right)^+\right)$
  \item % (d) The function
    $\Prob_{(x,y),\tau}\left(\Gamma^{\epsilon}(x)\setminus\{y\}\ne\emptyset\right)$
  \end{enumerate}
  are right-continuous in $\epsilon$.
\end{corollary}
\begin{proof}
  All these statements can be deduced from part~(a) of Lemma~\ref{lem:right-continuous}
  in the same way as in the proof of part~(b) of that lemma.
  The right-continuity of the function $\Gamma^{\epsilon}(x,\tau)$
  implies the right-continuity of $\III_{\{\left|\Gamma^{\epsilon}(x)\right|>1\}}$
  (remember that $\left|\Gamma^{\epsilon}(x)\right|$ takes only integer values).
  Therefore, the right-continuity of
  $\Prob_{x,\tau}\left(\left|\Gamma^{\epsilon}(x)\right|>1\right)$
  follows by the Lebesgue dominated convergence theorem.
  This proves (a), and proofs of (b)--(d) are analogous.
\end{proof}

First suppose that $A$ is M-optimal.
Let $B$ be any idealised conformity measure.
From \eqref{eq:Ilia-SP-1},
it is clear that \eqref{eq:U-optimal-1} holds with $<$ replaced by $\le$.
If, furthermore, we have \eqref{eq:U-optimal-2}:
by Corollary~\ref{cor:right-continuous} we also have \eqref{eq:M-optimal-2} for all $\epsilon$;
therefore, we also have \eqref{eq:ME-optimal-3} for all $\epsilon$;
in combination with \eqref{eq:Ilia-SP-2}, we obtain \eqref{eq:UF-optimal-3}.
Therefore, $A$ is U-optimal.

Now suppose that $A$ is U-optimal.
Let $B$ be the SP idealised conformity measure,
which we know to be not only M-optimal but also U-optimal
(as shown in the previous paragraph).
By the definition (\eqref{eq:U-optimal-1}--\eqref{eq:UF-optimal-3}) of U-optimality,
we have \eqref{eq:U-optimal-2} and \eqref{eq:UF-optimal-3} with $=$ in place of $\le$.
This implies that \eqref{eq:M-optimal-2} holds for all $\epsilon$
(had the equality been violated for some $\epsilon\in(0,1)$,
it would have been violated for a range of $\epsilon$ by Corollary~\ref{cor:right-continuous},
which would have contradicted \eqref{eq:U-optimal-2}).
In the same way, it implies that \eqref{eq:ME-optimal-3} holds
(even with $=$ in place of $\ge$)
for all $\epsilon$.
Therefore, $A$ is M-optimal.

\subsection{Proof of Theorem~\protect\ref{thm:MCP}}

Our argument for $\mathcal{O}(\EEE)=\mathcal{R}(\MCP)$
will be similar to the argument for $\mathcal{O}(\MMM)=\mathcal{R}'(\SP)$
given in the previous subsection;
we will again analyze the requirements imposed by being E-optimal
starting from small values of $\epsilon\in(0,1)$.
Let $g_1<g_2<\cdots<g_n$ be the list of the conditional probabilities $Q(y\mid x)$
of all examples $(x,y)\in\mathbf{Z}$,
with all duplicates removed and the remaining conditional probabilities sorted in the increasing order.
All examples will be split into $2n$ groups,
with the examples in the $i$th and $(n+i)$th groups satisfying $Q(y\mid x)=g_i$,
$i=1,\ldots,n$.
Initially the $i$th group, $i=1,\ldots,n$,
contains all examples satisfying $Q(y\mid x)=g_i$,
and the other groups are empty.
(Later some of the examples will be moved into the groups numbered $n+1,n+2,\ldots$,
and as a result some of the first $n$ groups may become empty.)
It will be true that each element of the $i$th group will be coming before each element of the $j$th group when $1\le i<j\le2n$.

Any F-optimal idealised conformity measure
will assign the lowest conformity to the first group of examples,
perhaps except for examples $(x,y)$ for which $Q(y\mid x)=\max_{y'}Q(y'\mid x)$.
If for some $x\in\mathbf{X}$, the first group contains $(x,y)$ with $Q(y\mid x)=\max_{y'}Q(y'\mid x)$,
we choose one such $(x,y)$ for each such $x$ and move it to the $(n+1)$th group.
The rest of the examples in the group can be ordered in their conformity in any way (with ties allowed).
The examples in the $(n+1)$th group can also be ordered arbitrarily.
Process the 2nd, 3rd,\ldots, $n$th groups in the same way.
It is clear that in the end we will obtain a refinement of an MCP idealised conformity measure.

Next we prove $\mathcal{O}(\EEE)=\mathcal{O}(\FFF)$.
Defining a \emph{p-choice function} $\tilde y:\mathbf{X}\to\mathbf{Y}$
(for a given idealised conformity measure)
by the requirement
\[
  p(x,\tilde y(x))
  =
  \max_y p(x,y),
\]
we have the following analogue of \eqref{eq:basic}:
\begin{multline*}
  \sum_{y \in \mathbf{Y}}
  p(x,y)
  -
  \max_{y \in \mathbf{Y}}
  p(x,y)
  =
  \sum_{y \in \mathbf{Y}\setminus\{\tilde y(x)\}}
  p(x,y)
  =
  \sum_{y \in \mathbf{Y}\setminus\{\tilde y(x)\}}
  \int_0^1
  \III_{\{p(x,y)>\epsilon\}}
  \dd{\epsilon}\\
  =
  \int_0^1
    \sum_{y \in \mathbf{Y}\setminus\{\tilde y(x)\}}
    \III_{\{p(x,y)>\epsilon\}}
  \dd{\epsilon}
  =
  \int_0^1
    \left(
      \left|\Gamma^{\epsilon}(x)\right|
      -
      1
    \right)^+
  \dd{\epsilon}.
\end{multline*}
This implies, similarly to \eqref{eq:expectation},
\begin{equation}\label{eq:Ilia-MCP}
  \Expect_{x,\tau}
  \left(
    \sum_{y \in \mathbf{Y}}
    p(x,y)
    -
    \max_{y \in \mathbf{Y}}
    p(x,y)
  \right)
  =
  \int_0^1
    \Expect_{x,\tau}
    \left(
    \left(
      \left|\Gamma^{\epsilon}(x)\right|
      -
      1
    \right)^+
    \right)
  \dd{\epsilon}.
\end{equation}

Suppose that $A$ is E-optimal, and let $B$ be any idealised conformity measure.
From \eqref{eq:Ilia-MCP},
it is clear that \eqref{eq:F-optimal-1} holds with $<$ replaced by $\le$.
If, furthermore, we have \eqref{eq:F-optimal-2}:
by Corollary~\ref{cor:right-continuous}(c) we also have \eqref{eq:E-optimal-2} for all $\epsilon$;
therefore, we also have \eqref{eq:ME-optimal-3} for all $\epsilon$;
in combination with \eqref{eq:Ilia-SP-2}, we obtain \eqref{eq:UF-optimal-3}.
Therefore, $A$ is F-optimal.

Now suppose that $A$ is F-optimal.
Let $B$ be any MCP idealised conformity measure,
which we know to be both E-optimal and F-optimal.
By the definition of F-optimality,
we have \eqref{eq:F-optimal-2} and \eqref{eq:UF-optimal-3} with $=$ in place of $\le$.
As in the previous subsection, this implies that \eqref{eq:E-optimal-2} holds for all $\epsilon$,
and also implies that \eqref{eq:ME-optimal-3} holds
(even with $=$ in place of $\ge$)
for all $\epsilon$.
Therefore, $A$ is E-optimal.

\subsection{Proof of Theorem~\protect\ref{thm:MSP}}

The proof is similar to the proofs in the previous two subsections.
First we check that $\mathcal{O}(\OM)=\mathcal{R}''(\MSP)$,
analyzing the requirement of OM-optimality
starting from small values of $\epsilon\in(0,1)$.
Let $f_1>f_2>\cdots>f_n>0.5$ be the list of the predictabilities of all objects $x\in\mathbf{X}$
whose predictability exceeds $0.5$,
with all duplicates removed and the remaining predictabilities sorted in the decreasing order.
All examples are split into $2n+1$ groups (perhaps some of them empty)
in such a way that each element of the $i$th group is coming before each element of the $j$th group
when $1\le i < j \le 2n+1$.
The $i$th group, $i=1,\ldots,n$,
contains all examples $(x,y)$ with predictability $f_i$ and $Q(y\mid x)<1/2$,
the $(n+1)$th group contains all examples with predictability $0.5$ or less,
and the $(n+1+i)$th group, $i=1,\ldots,n$,
contains all examples $(x,y)$ with $Q(y\mid x)=f_i$
(there is, however, at most one such example);
it is possible that $n=0$.

Any OM-optimal idealised conformity measure will assign the lowest conformity
to the first group of examples (assuming $n\ge1$),
and those examples can be ordered arbitrarily in their conformity,
except that any examples sharing their objects should have the same conformity.
This group takes care of the values
\[
  \epsilon
  \in
  \left(
    0,
    Q
    \left\{
      (x,y) \mid f(x)=f_1 \And Q(y\mid x) \ne f_1
    \right\}
  \right].
\]
Proceed in the same way through groups $2,\ldots,n$.
The $(n+1)$th group is most complicated (when non-empty).
It contains the following kinds of examples:
\begin{itemize}
\item
  Examples whose predictability is less than $0.5$.
  All such examples should have the same conformity if they share the same object.
\item
  Examples $(x,y)$ whose predictability is exactly $0.5$
  and which satisfy $Q(y\mid x)<0.5$.
  All such examples should have the same conformity if they share the same object.
\item
  Examples $(x,y)$ whose predictability is exactly $0.5$
  and which satisfy $Q(y\mid x)=0.5$.
\end{itemize}
Otherwise, the examples in the $(n+1)$th group can be ordered arbitrarily
in their conformity.
Groups $n+2,\ldots,2n+1$ are singletons or empty and do not cause any problems.
Therefore, an idealised conformity measure is OM-optimal if and only if it is in $\mathcal{R}''(\MSP)$.

Next we check that $\mathcal{O}(\OU)=\mathcal{O}(\OM)$.
Similarly to \eqref{eq:Ilia-SP-1},
we have, for a given idealised conformity measure,
\begin{multline}\label{eq:Ilia-MSP}
  \Expect_{(x,y),\tau}
  \max_{y'\ne y}
  p(x,y',\tau)
  =
  \Expect_{(x,y),\tau}
  \int_0^1
  \III_
  {\{
    \max_{y'\ne y}
    p(x,y',\tau)
    >
    \epsilon
  \}}
  \dd{\epsilon}\\
  =
  \Expect_{(x,y),\tau}
  \int_0^1
  \III_
  {\{
    \Gamma^{\epsilon}(x)
    \setminus
    \{y\}
    \ne
    \emptyset
  \}}
  \dd{\epsilon}
  =
  \int_0^1
    \Prob_{x,\tau}
    \left(
      \Gamma^{\epsilon}(x)
      \setminus
      \{y\}
      \ne
      \emptyset
    \right)
  \dd{\epsilon}.
\end{multline}

By \eqref{eq:Ilia-MSP},
OM-optimality immediately implies OU-optimality.

Now suppose that $A$ is OU-optimal.
Let $B$ be the MSP idealised conformity measure,
which is both OM-optimal and OU-optimal.
If \eqref{eq:OM-optimal} is violated for some $\epsilon$,
it is violated for a range of $\epsilon$
(by Corollary~\ref{cor:right-continuous}(d)),
which, by \eqref{eq:Ilia-MSP}, contradicts the OU-optimality of $A$.
Therefore, $A$ is OM-optimal.

\section{Empirical Study}
\label{sec:empirical}

In this section we demonstrate differences between two of our $\epsilon$-free criteria,
OF (probabilistic) and U (standard but not probabilistic) on the USPS data set of hand-written digits
(\cite{LeCun/etal:1990}; examples of such digits are given in Figure~\ref{fig:images},
which is a subset of Figure~2 in \cite{LeCun/etal:1990}).
We use the original split of the data set into the training and test sets.
Our programs are written in R,
and the results presented in the figures below are for the seed $0$ of the R random number generator;
however, we observe similar results in experiments with other seeds.

\begin{figure*}[t]
\begin{center}
  \includegraphics[width=0.3\textwidth]{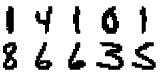}
\end{center}
\caption{Examples of hand-written digits in the USPS data set.}
\label{fig:images}
\end{figure*}

The problem is to classify hand-written digits,
the labels are elements of $\{0,\ldots,9\}$,
and the objects are elements of $\mathbb{R}^{256}$,
where the $256$ numbers represent the brightness of pixels in $16\times16$ pictures.
We normalise each object
by applying the same affine transformation (depending on the object) to each of its pixels
making the mean brightness of the pixels in the picture equal to $0$
and making its standard deviation equal to $1$.
The sizes of the training and test sets are $7291$ and $2007$, respectively.

\begin{figure*}[t]
\begin{center}
  \includegraphics[width=\textwidth]{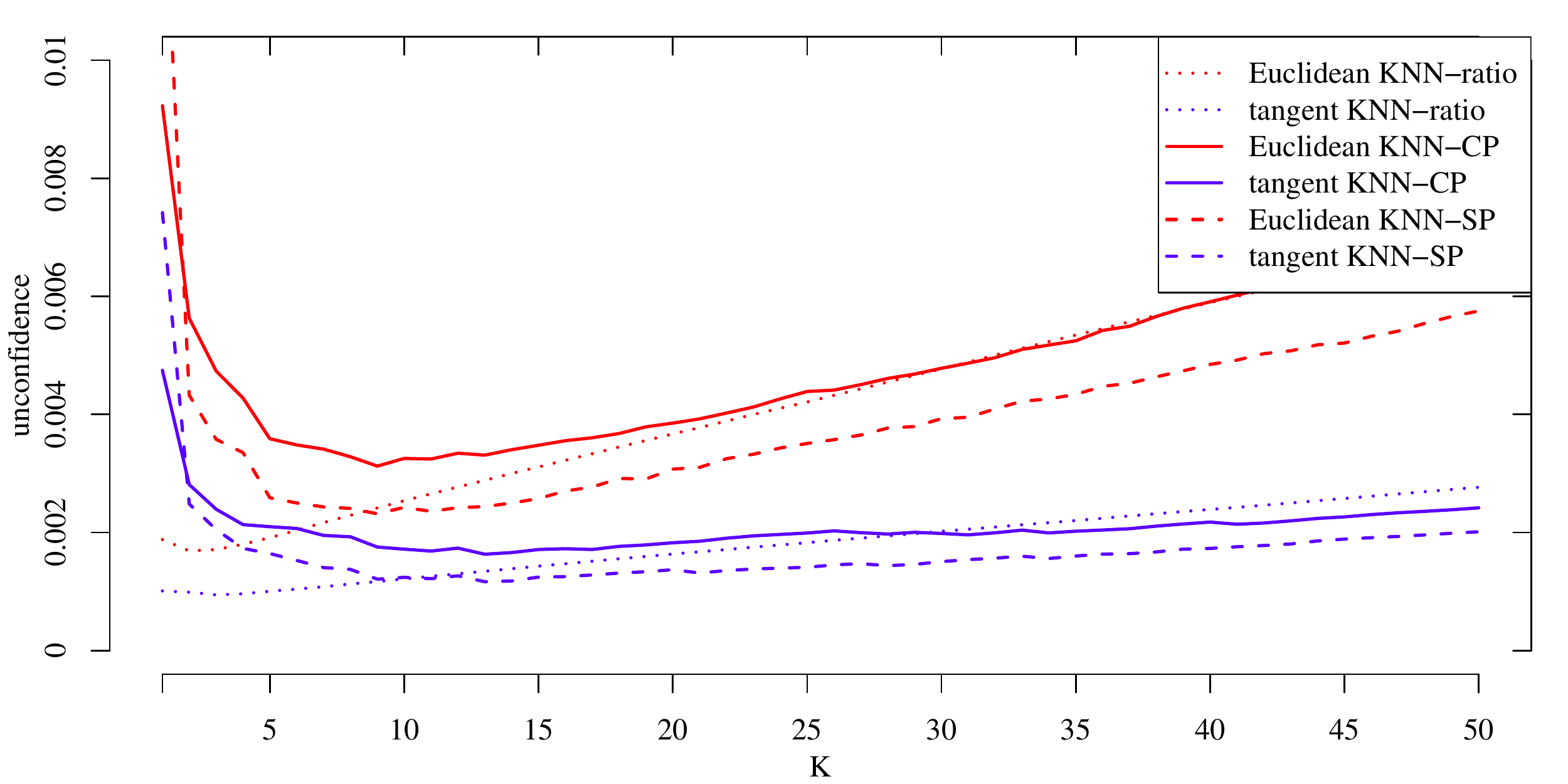}\\[2mm]
  \includegraphics[width=\textwidth]{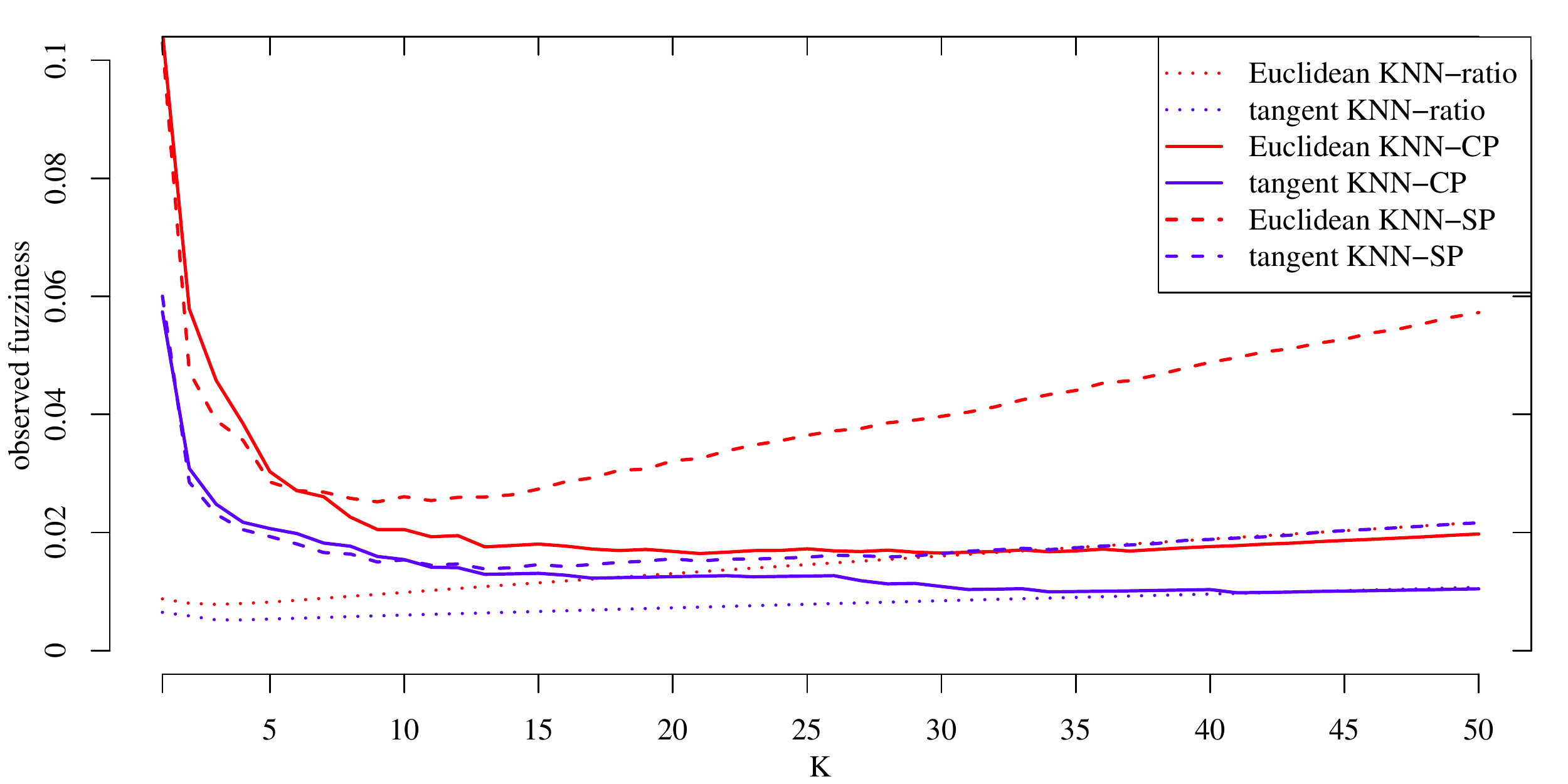}
\end{center}
\caption{Top plot: average unconfidence for the USPS data set
  (for different values of parameters).
  Bottom plot: average observed fuzziness for the USPS data set.
  In black-and-white the lines of the same type (dotted, solid, or dashed)
  corresponding to Euclidean and tangent distances
  can always be distinguished by their position:
  the former is above the latter.}
\label{fig:USPS}
\end{figure*}

We evaluate six conformal predictors using the two criteria of efficiency.
Fix a metric on the object space $\mathbb{R}^{256}$;
in our experiments we use tangent distance (as implemented by Daniel Keysers) and Euclidean distance.
Given a sequence of examples $(z_1,\ldots,z_n)$, $z_i=(x_i,y_i)$,
we consider the following three ways of computing conformity scores:
for $i=1,\ldots,n$,
\begin{itemize}
\item
  $\alpha_i := \sum_{j=1}^K d^{\ne}_j / \sum_{j=1}^K d^{=}_j$,
  where $d_j^{\ne}$ are the distances, sorted in the increasing order,
  from $x_i$ to the objects in $(z_1,\ldots,z_n)$
  with labels different from $y_i$
  (so that $d_1^{\ne}$ is the smallest distance from $x_i$ to an object $x_j$ with $y_j\ne y_i$),
  and $d_j^{=}$ are the distances, sorted in the increasing order,
  from $x_i$ to the objects in $(z_1,\ldots,z_{i-1},z_{i+1},\ldots,z_n)$
  labelled as $y_i$
  (so that $d_1^{=}$ is the smallest distance from $x_i$ to an object $x_j$ with $j\ne i$ and $y_j=y_i$).
  We refer to this conformity measure as the \emph{KNN-ratio conformity measure};
  it has one parameter, $K$,
  whose range is $\{1,\ldots,50\}$ in our experiments
  (so that we always have $K\ll n$).
\item
  $\alpha_i:=N_i/K$,
  where $N_i$ is the number of objects labelled as $y_i$ among the $K$ nearest neighbours of $x_i$
  (when $d_K=d_{K+1}$ in the ordered list $d_1,\ldots,d_{n-1}$ of the distances from $x_i$
  to the other objects,
  we choose the nearest neighbours randomly among $z_j$ with $y_j=y_i$ and with $x_j$ at a distance of $d_K$ from $x_i$).
  This conformity measure is a KNN counterpart of the CP idealised conformity measure (cf.\ \eqref{eq:CP}),
  and we will refer to it as the \emph{KNN-CP conformity measure};
  its parameter $K$ is in the range $\{2,\ldots,50\}$ in our experiments.
\item
  finally, we define $f_i:=\max_y(N_i^y/K)$,
  where $N_i^y$ is the number of objects labelled as $y$ among the $K$ nearest neighbours of $x_i$,
  $\hat y_i\in\arg\max_y(N_i^y/K)$
  (chosen randomly from $\arg\max_y(N_i^y/K)$ if $\left|\arg\max_y(N_i^y/K)\right|>1$),
  and
  \begin{equation*}
    \alpha_i
    :=
    \begin{cases}
      f_i & \text{if $y_i = \hat y_i$}\\
      -f_i & \text{otherwise};
    \end{cases}
  \end{equation*}
  this is the \emph{KNN-SP conformity measure}.
\end{itemize}
The three kinds of conformity measures
combined with the two metrics (tangent and Euclidean) give six conformal predictors.

Figure~\ref{fig:USPS} gives the average unconfidence~\eqref{eq:U}
(top panel) and the average observed fuzziness~\eqref{eq:OF}
(bottom panel)
over the test sequence (so that $k=2007$)
for a range of the values of the parameter $K$.
Each of the six lines corresponds to one of the conformal predictors,
as shown in the legends;
in black-and-white the lines of the same type (dotted, solid, or dashed)
corresponding to Euclidean and tangent distances
can always be distinguished by their position:
the former is above the latter.

The best results are for the KNN-ratio conformity measure combined with tangent distance
for small values of the parameter $K$.
For the two other types of conformity measures their relative evaluation changes
depending on the kind of a criterion used to measure efficiency:
as expected, the KNN-CP conformal predictors are better under the OF criterion,
whereas the KNN-SP conformal predictors are better under the U criterion
(cf.\ Theorems~\ref{thm:CP} and~\ref{thm:SP}),
if we ignore small values of $K$
(when the probability estimates $N_i^y/K$ are very unreliable).

Our conclusion is that whereas some conformal predictors (such as the KNN-ratio ones in our experiments)
can perform well under different criteria of efficiency,
the performance of other conformal predictors depends very much on the criterion of efficiency
used to evaluate it.

\section{Efficiency of Label-conditional Conformal Predictors and Transducers}
\label{sec:label-conditional}

Conformal predictors, as defined in Section~\ref{sec:criteria},
only guarantee the overall coverage probability,
averaged over all labels.
Sometimes we want to have a guarantee for the coverage probability for each label $y\in\mathbf{Y}$ separately,
and in this case one should use label-conditional conformal predictors,
which are studied in this section.

\subsection{Label-conditional conformal predictors and transducers}

The \emph{label-conditional conformal predictor} determined by a conformity measure $A$
is defined by \eqref{eq:conformal-predictor}
where the \emph{label-conditional p-values} $p^y$ are defined by
\begin{multline}\label{eq:p-label-conditional}
  p^y
  :=
  \Bigl(
    \left|\left\{i=1,\ldots,l\mid y_i=y \And \alpha^y_i<\alpha^y_{l+1}\right\}\right|\\
    +
    \tau
    \left|\left\{i=1,\ldots,l\mid y_i=y \And \alpha^y_i=\alpha^y_{l+1}\right\}\right|
    +
    \tau
  \Bigr)\\
  /
  \left(
    \left|\left\{i=1,\ldots,l\mid y_i=y\right\}\right|
    +
    1
  \right)
\end{multline}
(instead of~\eqref{eq:p});
as before,
$\tau$ is a random number distributed uniformly on the interval $[0,1]$
(conditionally on all the examples),
and the conformity scores are defined by~\eqref{eq:conformity-scores}.

The \emph{label-conditional conformal transducer} determined by $A$
outputs the system of p-values $(p^y\mid y\in\mathbf{Y})$
defined by \eqref{eq:p-label-conditional}
for each training sequence $(z_1,\ldots,z_l)$ of examples and each test object $x$.
The property of validity for label-conditional conformal predictors and transducers
is that the p-values $p^y$ are distributed uniformly on $[0,1]$ given $y$
when the examples $z_1,\ldots,z_l,(x,y)$ are generated independently
from the same probability distribution $Q$ on $\mathbf{Z}$
(see, e.g., \cite{Vovk/etal:2005book}, Proposition~4.10).
This implies that the conditional probability of error,
$y\notin\Gamma^{\epsilon}(z_1,\ldots,z_l,x)$,
given $y$ is $\epsilon$ at any significance level~$\epsilon$.

The p-values \eqref{eq:p-label-conditional},
and the corresponding conformal predictors and transducers,
only depend on the conformity order within each class:
now we define $(x_i,y_i)\preceq(x_j,y_j)$ to mean $y_i=y_j$ and $\alpha_i\le\alpha_j$
(with $(x_i,y_i)$ and $(x_j,y_j)$ such that $y_i\ne y_j$ being incomparable).

The definitions of all ten criteria of efficiency
introduced in Section~\ref{sec:criteria} and listed in Table~\ref{tab:criteria}
carry over to the case of label-conditional conformal transducers and predictors.

\subsection{Idealised setting}

As before, we assume that the object space $\mathbf{X}$ is finite
and $Q_{\mathbf{X}}(x)>0$ for all $x\in\mathbf{X}$.
We also assume $Q_{\mathbf{Y}}(y)>0$ for all $y\in\mathbf{Y}$,
where $Q_{\mathbf{Y}}$ is the marginal distribution of $Q$ on the label space $\mathbf{Y}$.

Let $A$ be an idealised conformity measure.
For each potential label $y\in\mathbf{Y}$ for an object $x$
define the corresponding \emph{label-conditional p-value} as
\begin{multline}\label{eq:p-value-label-conditional}
  p^y
  =
  p(x,y)
  :=
  \frac{Q\{(x',y)\in\mathbf{Z}\mid A(x',y)<A(x,y)\}}{Q_{\mathbf{Y}}(y)}\\
  +
  \tau
  \frac{Q\{(x',y)\in\mathbf{Z}\mid A(x',y)=A(x,y)\}}{Q_{\mathbf{Y}}(y)},
\end{multline}
analogously to \eqref{eq:p-value},
with the same random number $\tau\in[0,1]$ used for all $(x,y)$.
The \emph{label-conditional idealised conformal predictor} is defined by \eqref{eq:prediction-set}
for the new definition of $p(x,y)$
and the \emph{label-conditional idealised conformal transducer}
corresponding to the idealised conformity measure $A$
outputs for each object $x\in\mathbf{X}$
the system of p-values $(p^y\mid y\in\mathbf{Y})$ defined by \eqref{eq:p-value-label-conditional}.

The idealised p-values \eqref{eq:p-value-label-conditional},
and the corresponding idealised conformal predictors and transducers,
also depend only on the conformity order within each class:
we can define $(x,y)\preceq(x',y')$ to mean $y=y'$ and $A(x,y)\le A(x',y')$.
Two idealised conformity measures are \emph{equivalent within classes}
if they lead to the same order $\preceq$;
in this section we will consider only this notion of equivalence
(without mentioning it explicitly).

The properties of validity now become conditional:
\begin{itemize}
\item
  If $(x,y)$ is generated from $Q$
  and $\tau$ is generated independently from the uniform probability distribution on $[0,1]$,
  $p(x,y)$ is distributed uniformly on $[0,1]$ even if we condition on $y$.
\item
  Therefore, at each significance level $\epsilon$
  the idealised conformal predictor makes an error
  with conditional probability $\epsilon$ given $y$.
\end{itemize}

\subsection{Probabilistic criteria of efficiency}

\emph{Label-conditionally S-optimal}, \emph{N-optimal}, \emph{OF-optimal}, and \emph{OE-optimal}
idealised conformity measures are defined exactly as S-optimal, N-optimal, OF-optimal, and OE-optimal
idealised conformity measures at the end of Section~\ref{sec:optimal}
but with the label-conditional definitions of the p-values and prediction sets.

Let us say that an idealised conformity measure $A$ is a \emph{label-conditional refinement}
of an idealised conformity measure $B$
if
\begin{equation*}
  B(x_1,y)<B(x_2,y)
  \Longrightarrow
  A(x_1,y)<A(x_2,y)
\end{equation*}
for all $x_1,x_2\in\mathbf{X}$ and all $y\in\mathbf{Y}$.
Notice that the notion of label-conditional refinement is weaker than that of refinement
(as defined by \eqref{eq:refinement}):
if $A$ is a refinement of $B$,
then $A$ is a label-conditional refinement of $B$
(but not vice versa, in general).
Let $\mathcal{R}\lc(\CP)$ be the set of all label-conditional refinements of the CP idealised conformity measure.
If $C$ is a criterion of efficiency (one of the ten criteria in Table~\ref{tab:criteria}),
we let $\mathcal{O}\lc(C)$ stand for the set of all label-conditionally $C$-optimal idealised conformity measures.
We have the following simple corollary of Theorem~\ref{thm:CP}.

\begin{theorem}\label{thm:CP-lc}
  $\mathcal{O}\lc(\SSS)=\mathcal{O}\lc(\OF)=\mathcal{O}\lc(\NNN)=\mathcal{O}\lc(\OE)=\mathcal{R}\lc(\CP)$.
\end{theorem}

\begin{proof}
  The proof is a modification of the proof of Theorem~\ref{thm:CP}.
  In the case of $\mathcal{O}\lc(\NNN)$,
  for each label $y\in\mathbf{Y}$ we have a separate optimization problem.
  Now the constraint becomes
  \begin{equation*}
    \sum_{x}
    Q(x,y)P(x,y)
    =
    \epsilon Q_{\mathbf{Y}}(y)
  \end{equation*}
  (in place of \eqref{eq:constraint}),
  and the objective becomes to maximise $\sum_x Q'(x,y)P(x,y)$
  (since maximising the sum over $(x,y)$ in \eqref{eq:objective}
  can be achieved by maximizing the sum over $x$ for each $y$ separately).
  Now an application of the Neyman--Pearson lemma, as in the proof of Theorem~\ref{thm:CP},
  shows that $\mathcal{O}\lc(\NNN)=\mathcal{R}\lc(\CP)$.

  The same argument as in the proof of Theorem~\ref{thm:CP}
  (the last three paragraphs) shows that
  $\mathcal{O}\lc(\NNN)=\mathcal{O}\lc(\SSS)=\mathcal{O}\lc(\OF)=\mathcal{O}\lc(\OE)$,
  and so we have
  the formula in Theorem~\ref{thm:CP-lc}.
\end{proof}

We say that an efficiency criterion is \emph{label-conditionally probabilistic}
if the CP idealised conformity measure is label-conditionally optimal for it;
we add the qualifier \emph{weakly} if this is true for some (label-conditional) refinement of CP
and \emph{strongly} if this is true for an arbitrary (label-conditional) refinement of CP.
We can see that the four criteria that are set in italics in Table~\ref{tab:criteria}
are still optimal in the label-conditional setting.

\subsection{Other criteria of efficiency}

Using the label-conditional definitions of the p-values and prediction sets,
we define \emph{label-conditionally U-optimal}, \emph{M-optimal}, \emph{F-optimal}, \emph{E-optimal},
\emph{OU-optimal}, and \emph{OM-optimal} idealised conformity measures
in exactly the same way as their unconditional counterparts
at the beginning of Section~\ref{sec:not-probabilistic}.
The label-conditional U and M criteria are standard, and the label conditional E criterion
(with a different treatment of empty observations)
has been introduced and explored in \cite{Sadinle/etal:2016}.

We do not give label-conditional analogues of Theorems~\ref{thm:SP}--\ref{thm:MSP},
since the label-conditionally U-, M-, F-, E-, OU-, and OM-optimal idealised conformity measures
are unlikely to have explicit \label{p:combinatorial-2}expressions
(cf.\ our remark about deterministic conformal predictors on p.~\pageref{p:combinatorial-1}),
unless $\left|\mathbf{Y}\right|=2$.
The following theorem says that all of these criteria are BW probabilistic
(and the examples that we will give after its proof will show that they are not probabilistic).

\begin{theorem}
  If $\left|\mathbf{Y}\right|=2$,
  each of the sets
  \begin{equation}\label{eq:binary-probabilistic-lc}
    \mathcal{O}\lc(\UUU), \mathcal{O}\lc(\MMM), \mathcal{O}\lc(\FFF), \mathcal{O}\lc(\EEE), \mathcal{O}\lc(\OU), \mathcal{O}\lc(\OM)
  \end{equation}
  contains a refinement of the CP idealised conformity measure.
\end{theorem}

\begin{proof}
  Assume, without loss of generality, that $\mathbf{Y}=\{0,1\}$.
  And let us assume, for simplicity, that the values $Q(1\mid x)$ are all different for different $x\in\mathbf{X}$
  (if this condition is not satisfied, the theorem still holds,
  but finding a suitable refinement becomes, in general, a difficult combinatorial problem).
  In this case it is easy to see that each of the sets in \eqref{eq:binary-probabilistic-lc}
  is the equivalence class of the CP idealised conformity measure:
  we can construct the optimal idealised conformity measure gradually starting from small values of $\epsilon$,
  as in the proofs of Theorems~\ref{thm:SP}--\ref{thm:MSP}.
\end{proof}

The following examples show that none of the criteria considered in this subsection
is probabilistic (or even weakly probabilistic):
\begin{itemize}
\item
  Let $\mathbf{X}=\{1,2\}$, $\mathbf{Y}=\{1,2,3,4\}$, and
  \begin{equation}\label{eq:lc-example}
    \begin{aligned}
      Q_{\mathbf{X}}(1)&=0.5 & Q(1\mid 1)&=0.2 & Q(2\mid 1)&=0.3 & Q(3\mid 1)&=0.2 & Q(4\mid 1)&=0.3 \\
      Q_{\mathbf{X}}(2)&=0.5 & Q(1\mid 2)&=0.3 & Q(2\mid 2)&=0.2 & Q(3\mid 2)&=0.3 & Q(4\mid 2)&=0.2
    \end{aligned}
  \end{equation}
  ($Q(y\mid x)$ meaning $Q_{\mathbf{Y}\mid\mathbf{X}}(y\mid x)$, as usual).
  All refinements of the CP idealised conformity measure are equivalent
  (as for different labels $y$ the two conditional probabilities $Q(y\mid x)$, $x=1,2$, are different),
  and so all of them will lead to the same p-values.
  Let $A$ be any idealised conformity measure
  that makes all observations containing object $1$ less conforming than all observations containing object $2$.
  The U criterion is not probabilistic since the expression \eqref{eq:U-expression}
  is $0.7$ for the CP idealised conformity measure
  and is smaller, $0.55$, for the idealised conformity measure $A$.
  The M criterion is not probabilistic since at significance level $\epsilon=0.4$
  the CP idealised conformity measure gives the predictor
  $\Gamma^{\epsilon}(1)=\{2,4\}$ and $\Gamma^{\epsilon}(2)=\{1,3\}$ (a.s.),
  and so
  \begin{equation*}
    \Prob_{x,\tau}(\left|\Gamma^{\epsilon}_{\CP}(x)\right|>1)
    = 1 > 2/3 =
    \Prob_{x,\tau}(\left|\Gamma^{\epsilon}_A(x)\right|>1)
  \end{equation*}
  (cf.\ \eqref{eq:M-optimal-1}).
\item
  Let $\mathbf{X}=\{1,2,3\}$, $\mathbf{Y}=\{1,2,3\}$, and, for a small $\delta>0$,
  \begin{align*}
    Q_{\mathbf{X}}(1)&=1/3 & Q(1\mid 1)&=1/3+\delta & Q(2\mid 1)&=1/3-2\delta & Q(3\mid 1)&=1/3+\delta \\
    Q_{\mathbf{X}}(2)&=1/3 & Q(1\mid 2)&=1/3-\delta & Q(2\mid 2)&=1/3+2\delta & Q(3\mid 2)&=1/3-\delta \\
    Q_{\mathbf{X}}(3)&=1/3 & Q(1\mid 3)&=1/3 & Q(2\mid 2)&=1/3 & Q(3\mid 3)&=1/3.
  \end{align*}
  All refinements of the CP idealised conformity measure are equivalent,
  and so the choice of the refinement does not affect the p-values.
  Let $A$ be an idealised conformity measure satisfying
  \begin{multline*}
    A(1,2)
    <
    A(2,1)
    =
    A(2,3)
    <
    A(3,1)
    =
    A(3,2)\\
    <
    A(1,1)
    =
    A(1,3)
    <
    A(2,2)
    <
    A(3,3)
  \end{multline*}
  (in other words, $A$ is the CP idealised conformity measure
  modified in such a way that that it assigns to $(3,3)$ the highest conformity score).
  The F criterion is not probabilistic since the expression \eqref{eq:F-expression}
  is $7/9+O(\delta)$ for the CP idealised conformity measure
  and is smaller (for sufficiently small $\delta$), $2/3+O(\delta)$, for $A$.
  The E criterion is not probabilistic since at significance level $\epsilon=2/3$
  the idealised conformity measure $A$ gives a predictor whose excess is always $0$,
  whereas the CP idealised conformity measure will have expected excess $1/3+O(\delta)$.
\item
  Let $\mathbf{X}=\{1,2\}$, $\mathbf{Y}=\{1,2,3,4\}$,
  and $Q$ be defined by \eqref{eq:lc-example}.
  Let $A$ be any idealised conformity measure
  that makes all observations containing object $1$ less conforming than all observations containing object $2$.
  The OU criterion is not probabilistic since the expression \eqref{eq:OU-expression}
  is $0.7$ for the CP idealised conformity measure
  and is smaller, $0.55$, for the idealised conformity measure $A$.
  The OM criterion is not probabilistic since at significance level $\epsilon=0.4$
  the CP idealised conformity measure produces an observed multiple prediction a.s.,
  whereas the idealised conformity measure $A$ produces an observed multiple prediction with probability $2/3$.
\end{itemize}

\section{Conclusion}
\label{sec:conclusion}

This paper investigates properties of various criteria of efficiency of conformal prediction
in the case of classification.
It would be interesting to transfer, to the extent possible, this paper's results to the cases of:
\begin{itemize}
\item
  Regression.
  The sum of p-values (as used in the S criterion)
  now becomes the integral of the p-value as function of the label $y$ of the test example,
  and the size of a prediction set becomes its Lebesgue measure
  (considered, as already mentioned, in \cite{Lei/Wasserman:2013}
  in the non-idealised case).
  Whereas the latter is typically finite,
  ensuring the convergence of the former is less straightforward.
\item
  Anomaly detection.
  A first step in this direction is made in \cite{Smith/etal:2014COPA},
  which considers the average p-value as its criterion of efficiency.
\item
  Infinite, including non-discrete, object spaces $\mathbf{X}$.
\item
  Non-idealised conformal predictors.
\item
  Significance levels $\epsilon=\epsilon_y$ that depend on the label $y\in\mathbf{Y}$
  in the label-conditional case.
\end{itemize}

The main part of this paper merely mentions what we called ``combinatorial problems''
(see pages \pageref{p:combinatorial-1} and \pageref{p:combinatorial-2}).
It would be interesting to explore them systematically.
As an example, let us consider the N criterion of efficiency for deterministic idealised conformal predictors
(with $\tau$ set to $1$ rather than being random)
in the case $\left|\mathbf{Y}\right|=1$
(which we did not allow in the main part of the paper;
in this case,
there is no difference between unconditional and label-conditional idealised conformal predictors).
The problem of finding an N-optimal idealised conformity measure
then becomes the \textsc{Subset-Sum Problem}, known to be NP-hard:
see, e.g., \cite{Martello/Toth:1990}, Chapter~4
(a special case of this problem, \textsc{Partition},
was already one of Karp's original 21 NP-complete problems \cite{Karp:1972}).
There are, however, efficient polynomial approximation schemes
for this problem.
It would be interesting, in particular, to find such schemes
for general deterministic idealised conformal predictors and transducers
and for smoothed idealised conformal predictors and transducers
for non-probabilistic criteria of efficiency in the label-conditional case.

\subsection*{Acknowledgments}

We are grateful to the reviewers of the conference version of this paper
for their helpful comments.
  This work was partially supported by EPSRC (grant EP/K033344/1),
  the Air Force Office of Scientific Research (grant ``Semantic Completions''),
  and the EU Horizon 2020 Research and Innovation programme (grant 671555).

\end{document}